\documentclass{article}

    \PassOptionsToPackage{numbers, compress}{natbib}

    \usepackage[final]{neurips_2022}

\usepackage[utf8]{inputenc} %
\usepackage[T1]{fontenc}    %
\usepackage{xcolor}         %
\definecolor{link-color}{RGB}{34, 180, 84} 
\definecolor{cite-color}{RGB}{46, 134, 190}
\usepackage[colorlinks=true,linkcolor=link-color,citecolor=cite-color]{hyperref}
\usepackage{url}            %
\usepackage{booktabs}       %
\usepackage{amsfonts}       %
\usepackage{nicefrac}       %
\usepackage{microtype}      %
\usepackage{xkcdcolors}
\definecolor{light-gray}{gray}{0.92}

\newcommand{\model}[0]{{Concrete-SM}}

\usepackage{amsthm}
\usepackage{thmtools} 
\usepackage{thm-restate}
\usepackage{wrapfig}

\usepackage[utf8]{inputenc} %
\usepackage[T1]{fontenc}    %
\usepackage{hyperref}       %
\usepackage{url}            %
\usepackage{booktabs}       %
\usepackage{amsfonts}       %
\usepackage{nicefrac}       %
\usepackage{microtype}      %

\usepackage{algorithmicx,algpseudocode,algorithm}

\usepackage{microtype}
\usepackage{graphicx}
\usepackage{booktabs} %
\usepackage{enumitem} 
\usepackage{bm}
\usepackage{nicefrac}
\usepackage{placeins}
\usepackage{pifont}%
\usepackage{wrapfig,tikz}
\usepackage{amsmath,longtable,fancyhdr,booktabs,multirow,graphicx,float}
\usepackage{adjustbox, bigstrut, tabularx, multirow, makecell, diagbox}
\usepackage{amssymb,xcolor}
\usepackage{color}
\usepackage{colortbl}
\usepackage{theoremref}
\usepackage{subcaption}
\usepackage{amsmath}

\usepackage{cleveref}
\usepackage{dsfont}

\newcommand{\norm}[1]{\left\lVert#1\right\rVert}

\newtheorem{lemma}{Lemma}
\newtheorem{theorem}{Theorem}

\newtheorem{example}{Example}

\newtheorem{proposition}{Proposition}

\newtheorem{definition}{Definition}

\newcommand{\be}{\begin{equation}}
	\newcommand{\ee}{\end{equation}}

\definecolor{Gray}{gray}{0.85}
\definecolor{LightCyan}{rgb}{0.88,1,1}

\makeatletter
\usepackage{xspace}
\def\@onedot{\ifx\@let@token.\else.\null\fi\xspace}
\DeclareRobustCommand\onedot{\futurelet\@let@token\@onedot}

\newcommand{\x}{\mathbf{x}}
\newcommand{\y}{\mathbf{y}}
\newcommand{\tx}{\tilde {\mathbf{x}}}

\def\ie{\emph{i.e}\onedot}

\def\vs{\emph{vs}\onedot}

\def\iid{i.i.d\onedot}

\usepackage[textsize=tiny]{todonotes}

\definecolor{blue1}{RGB}{0,128,255}
\definecolor{blue3}{RGB}{0,0,128}
\definecolor{darkpastelgreen}{rgb}{0.01, 0.75, 0.24}
\definecolor{cerulean}{rgb}{0.0, 0.48, 0.65}

\def\ie{\emph{i.e}\onedot}

\def\vs{\emph{vs}\onedot}

\def\iid{i.i.d\onedot}

\definecolor{darkgreen}{rgb}{0,0.6,0}

\def\eqref#1{equation~\ref{#1}}

\def\1{\bm{1}}

\def\rve{{\mathbf{e}}}

\def\rvv{{\mathbf{v}}}

\def\rvx{{\mathbf{x}}}

\def\vone{{\bm{1}}}

\def\vc{{\bm{c}}}

\def\vs{{\bm{s}}}

\DeclareMathAlphabet{\mathsfit}{\encodingdefault}{\sfdefault}{m}{sl}
\SetMathAlphabet{\mathsfit}{bold}{\encodingdefault}{\sfdefault}{bx}{n}

\def\gJ{{\mathcal{J}}}

\def\gN{{\mathcal{N}}}

\def\gP{{\mathcal{P}}}

\newcommand{\supp}{\rm{supp}}
\newcommand{\pdata}{p_{\rm{data}}}
\newcommand{\cdata}{\vc_{\pdata}}

\newcommand{\R}{\mathbb{R}}

\DeclareMathOperator*{\argmin}{arg\,min}

\newcommand{\bbR}{{\mathbb{R}}}
\newcommand{\bbE}{{\mathbb{E}}}

\newcommand{\cJ}{{\mathcal{J}}}
\newcommand{\cL}{{\mathcal{L}}}
\newcommand{\cN}{{\mathcal{N}}}
\newcommand{\cX}{{\mathcal{X}}}

\newcommand{\cG}{{\mathcal{G}}}

\usepackage{cleveref}

\title{Concrete Score Matching: Generalized Score Matching for Discrete Data}

\author{%
  Chenlin Meng\thanks{Joint first author}\\
  Stanford University\\
  \texttt{chenlin@cs.stanford.edu} \\
   \And
   Kristy Choi$^{*}$\\
   Stanford University \\
   \texttt{kechoi@cs.stanford.edu} \\
   \AND
   Jiaming Song \\
   NVIDIA \\
   \texttt{jiamings@nvidia.com} \\
   \And
   Stefano Ermon \\
   Stanford University, CZ Biohub \\
   \texttt{ermon@cs.stanford.edu} \\
}

\begin{document}

\maketitle
\begin{abstract}

\looseness=-1
Representing probability distributions by the gradient of their 
density functions 
has proven effective in modeling a wide range of continuous data modalities.
However, this representation is not applicable in discrete domains where the gradient 
is undefined.
To this end, we propose an analogous score function called the ``Concrete score'', a generalization of the (Stein) score for discrete settings.
Given a predefined neighborhood structure,
the Concrete score of any input is defined by the rate of change of the probabilities with respect to local directional changes of the input.
This formulation allows us to recover the (Stein) score in continuous domains
when measuring such changes by the Euclidean distance, while using the Manhattan distance leads to our novel score function in discrete domains.
Finally, we introduce a new framework to learn such scores from samples called Concrete Score Matching (CSM), and propose an efficient training objective to scale our approach to high dimensions.
Empirically, we demonstrate the efficacy of CSM on density estimation tasks on a mixture of synthetic, tabular, and high-dimensional image datasets, and demonstrate that it performs favorably relative to existing baselines for modeling discrete data.

\end{abstract}

\section{Introduction}
When estimating a statistical model, the representation of the underlying probability distribution has profound implications on the downstream modeling task. 
Likelihood-based model families such as Variational Autoencoders (VAEs)
\cite{kingma2013auto,rezende2014stochastic,poli2022self,higgins2016beta,van2017neural}, autoregressive models~\cite{van2016conditional,salimans2017pixelcnn++,meng2021improved,song2021accelerating,germain2015made}, normalizing flows \cite{rezende2015variational,kingma2016improved,papamakarios2017masked,meng2022butterflyflow,meng2020gaussianization,kingma2018glow,ho2019flow++,song2019mintnet}, and diffusion models \cite{sohl2015deep,ho2020denoising,sinha2021d2c,song2021denoising,meng2021sdedit,song2020score,su2022dual} are forced to either approximate the intractable normalizing constant or utilize restrictive model architectures; implicit models \cite{goodfellow2014generative,arjovsky2017wasserstein,gulrajani2017improved,karras2019style} try to capture the underlying sampling process by relying on unstable adversarial training procedures. 
On the other hand, representing the distribution as the gradient of the log probability density function---also known as the (Stein) score---allows us to circumvent such issues.
This is key to score matching's success on a wide range of continuous data modalities \cite{hyvarinen2005estimation,vincent2011connection,song2020sliced,meng2021estimating,meng2020autoregressive,choi2022density}.
In fact, its recent resurgence has led to significant advances in machine learning applications of density estimation, such as image generation \cite{song2019generative,song2020score,meng2021sdedit,ho2020denoising} and audio synthesis \cite{chen2020wavegrad,kong2020diffwave}, among others. 

However, score matching approaches that rely on modeling the gradient of the data distribution are inherently designed for continuous data; such methods hinge on the \emph{existence} of the gradient,
which is undefined for discrete domains \cite{vincent2011connection}. 
Given the 
ubiquitous nature of structured, discrete data in our world today such as graphs, text, genomic sequences, images, we
desire an approach that would allow us to expand the successes of score-based generative models into discrete domains.

\looseness=-1
To address this challenge, we first introduce the Concrete score: a generalization of the (Stein) score that is amenable to both \underline{con}tinuous and dis\underline{crete} data types\footnote{We borrow the terminology from ``concrete mathematics''~\cite{graham1989concrete}.}. Given a predefined neighborhood structure, the Concrete score leverages structural information in the data to construct \emph{surrogate gradient information} about the discrete space by considering the similarity between two neighboring examples. 
More precisely, the Concrete score of any input is defined by the rate of change of the probabilities with respect to directional changes of the input.
In direct analogy to the continuous (Stein) score, the intuition is that the Concrete score should remain small when two examples are ``close'' to one another; when the two examples are very different, the Concrete score will be large.
We find that measuring such changes by the Euclidean distance recovers the (Stein) score in continuous domains, while using the Manhattan distance leads to our novel score function in discrete domains.

Using this definition of the Concrete score, we then introduce a framework to learn such scores from samples called Concrete Score Matching (CSM). To successfully scale our method to high dimensions, we also propose an efficient training objective as well as several variations of our approach that improve performance in practice.
Empirically, we demonstrate the efficacy of CSM on a wide range of density estimation tasks on a mixture of synthetic, tabular, and high-dimensional image datasets, and demonstrate that it performs favorably relative to existing baselines for modeling discrete data.

The contributions of our work can be summarized as follows:
\begin{enumerate}
    \item We propose the Concrete score, a generalization of the (Stein) score to discrete domains. We construct the Concrete score such that it leverages structural information in the data to capture surrogate ``gradient'' information over discrete spaces.
    \item We introduce a novel score matching framework for estimating such scores from samples called Concrete Score Matching (CSM). We then outline several connections between CSM and existing (continuous) score matching techniques, such denoising score matching (DSM). 
    \item We propose efficient learning objectives that allow our method to scale gracefully to high-dimensional datasets, and show how to port over the recent successes in continuous score matching approaches to discrete domains.
\end{enumerate}

\section{Preliminaries}
\label{sec:cont_sm}
\looseness=-1
Let $\pdata(\rvx)$ be the unknown data distribution over $\mathcal{X} \in \mathbb{R}^D$,
for which we have access to \iid samples $\{\rvx_{i}\}_{i=1}^N \sim \pdata(\rvx)$. The (Stein) score of $\pdata(\rvx)$ is the first order derivative of the log data density function $\vs(\rvx)  = \nabla_\rvx \log p(\rvx)$. 
The goal of score matching is to learn an unnormalized density $q_\theta(\rvx)$ indexed by $\theta \in \Theta$ such that the estimated score function $\vs_\theta(\rvx) = \nabla_\rvx \log q_{\theta}(\rvx)$ is close to $\vs(\rvx)$.

In practice, we leverage a \emph{score network} $\vs_\theta: \mathbb{R}^D \rightarrow \bbR^D$ that is trained to minimize the Fisher divergence between $q_\theta(\rvx)$ and $\pdata(\rvx)$ \cite{song2020sliced,song2019generative}.
Although the original objective is intractable to compute due to its dependence on the ground truth data scores $\vs(\rvx)$, we can leverage integration by parts \cite{hyvarinen2005estimation} to simplify the objective as follows:
\begin{align}
\label{eq:sm_objective}
    \cJ_{SM}(\theta) 
    &= \bbE_{\pdata(\rvx)} \bigg[ \frac{1}{2} \Vert \vs_\theta(\rvx) \Vert_2^2 + \textrm{tr}(\nabla_\rvx \vs_\theta(\rvx)) \bigg] + \textrm{const.}
\end{align}
where $\textrm{tr}(\cdot)$ denotes the matrix trace, and the optimal score model satisfies $\vs_{\theta^*}(\rvx) \approx \vs(\rvx)$. 
While the trace term in Eq.~\ref{eq:sm_objective} remains problematic---it requires an expensive evaluation of the Hessian of the log-density function---recent works \cite{song2020sliced,song2020score} leverage the Skilling-Hutchinson trace estimator \cite{skilling1989eigenvalues,hutchinson1989stochastic} or directional derivatives \cite{pang2020efficient} to efficiently approximate the training objective.

Despite their key role in successfully scaling score-based generative models to high-dimensional datasets, we note two critical limitations of existing score matching techniques: (1) $\cX$ must be continuous; and (2) the density function $\pdata(\rvx)$ must be differentiable.
This poses a significant challenge in discrete domains, where both requirements fail to hold.

\section{The Concrete Score}
\label{sec:method}
The above limitations prevent the direct application of score matching techniques to discrete data. We thus propose the Concrete score, a generalization of the (Stein) score for discrete settings.
The key intuition is that although the gradient is undefined in discrete spaces, we can still construct a \emph{surrogate} 
for the gradient by leveraging local directional changes to the input. 
We do so by exploiting special neighborhood structures in the data, and elaborate upon the necessary conditions on these structures to guarantee that our surrogate gradient indeed recovers a valid score function that (1) completely characterizes the distribution, and (2) is amenable for parameter estimation.

\subsection{Constructing Surrogate Gradients for Discrete Data}
\label{sec:surrogate}
To be more precise, let $\pdata(\x)$ be the data
distribution over $\cX$.
We denote $\cN: \cX \to \cX^K$
as the function mapping each example $\rvx \in \cX$ to a set of neighbors, such that $\mathcal{N}(\x)=\{\x_{n_1},...,\x_{n_k}\}$ and $K = \vert \cN(\x) \vert$ \footnote{We consider a fixed number of neighbors $K$ for each $\rvx$ to simplify our notation, but note that this can be generalized to a variable number of neighbors for every $\rvx$.}. 
This neighborhood induces a particular graphical structure onto the support of $\pdata(\rvx)$, which we call the "neighborhood-induced graph", that will play a key role in constructing the surrogate gradient.
We provide a formal definition below.

\begin{definition}[Neighborhood-induced graph]
Let $\pdata(\x)$ be the data distribution over $\cX$ and $\cN$ be the function mapping each node $\x \in \cX$ to its set of neighbors.
The neighborhood-induced graph $\cG$ is the directed graph which results from adding a directed edge from $\rvx$ to each node in its neighborhood set $\x_n \in \cN(\rvx)$, for all $\rvx \in \supp(\pdata(\rvx))$.
\end{definition}
An important point is that the neighborhood structure can be asymmetric, since the neighborhood-induced graph is a directed graph. This implies that there may exist cases where $\cN(\rvx) = \{\rvx_1\}$ does not necessarily imply that $\cN(\rvx_1) = \{\rvx\}$. We provide an intuitive example below.
\begin{example}
When $\mathcal{X}=\{\x_0,\x_1,\ldots,\x_5\}$, and the neighborhood structure $\mathcal{N}$ is defined as: $\mathcal{N}(\x_0)=\{ \emptyset \}$ and $\mathcal{N}(\x_i)=\{\x_0\} \textrm{ } \forall i=1,\ldots,5$, the neighborhood-induced graph is the star graph in Figure~\ref{fig:graph_example}.
\end{example}

There exist a wide range of neighborhood structures, all of which yield different neighborhood-induced graphs. 
We visualize five common structures in 
Figure~\ref{fig:graph_example}.

\begin{figure}[H]
    \centering
    \includegraphics[width=\linewidth]{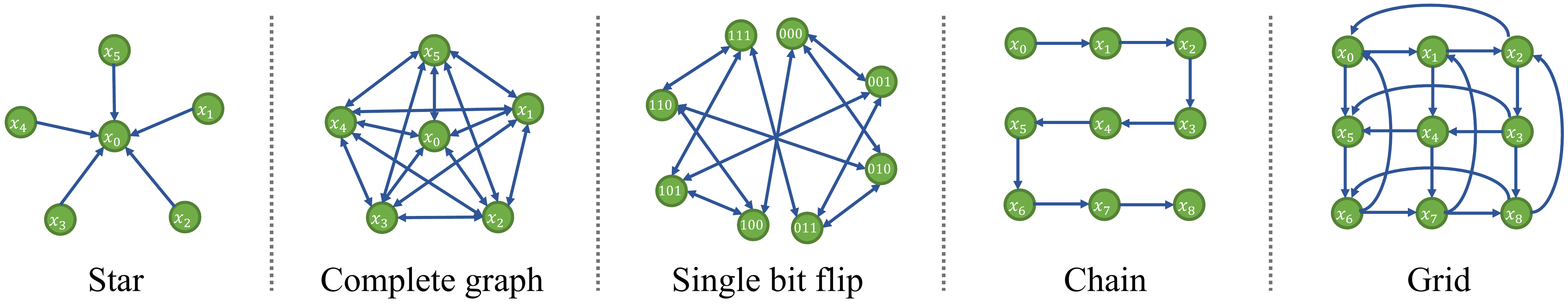}
    \caption{Examples of common neighborhood structures and their corresponding (connected) neighborhood-induced graphs, where the arrows point to the neighbors of a given node. Note that the neighborhood structure is directed.
    }
    \label{fig:graph_example}
    \vspace{-15pt}
\end{figure}

Such graphs provide insight into the kinds of neighborhood structures that are amenable to our framework.
One necessary characteristic of such graphs is \emph{connectedness}.
In fact, the five common graphical structures (among others) shown in Figure~\ref{fig:graph_example} are all weakly connected graphs.
We emphasize that connectedness does not take the directionality of the underlying graph $\cG$'s edges into account---it does not matter whether a node $\rvx$ has an incoming or outgoing edge.

With the above definitions in place, we can now define the surrogate gradient as the local differences between a set of examples
(similar to directional derivative in the continuous case).
Using this notion of the ``gradient'', we construct the Concrete score as the rate of change of the probabilities with respect to these local directional changes in the input $\rvx$ as defined below.

\begin{definition}[Concrete score]
\label{def:concrete}
Let $\cN$ be a function mapping each point $\rvx$ to its set of neighbors $\cN(\rvx) = \{\x_{n_1},...,\x_{n_k}\}$. %
Then, the Concrete score $\cdata(\rvx; \cN): \cX \rightarrow \bbR^{|\mathcal{N}(\x)|}$ for a given distribution $\pdata(\rvx)$ evaluated at $\rvx$ is:
\begin{equation}
\label{def:concrete_score}
    \cdata(\x ;\mathcal{N}) \triangleq \bigg[ 
    \frac{\pdata(\x_{n_1})-\pdata(\x)}{\pdata(\x)},
    ...,
    \frac{\pdata(\x_{n_k})-\pdata(\x)}{\pdata(\x)}\bigg]^T.
\end{equation}
\end{definition}

Although we have defined the Concrete score, we have yet to justify whether it is indeed a suitable score function for estimating $\pdata(\rvx)$.
The necessary condition, which we call \emph{completeness} \cite{lyu2012interpretation}, ensures that the Concrete score preserves enough information about $\pdata(\rvx)$ such that we can successfully learn $\pdata(\rvx)$ from data samples using score matching. 
We make this statement more precise in the following theorem.
\begin{restatable}[Completeness]{theorem}{completeness}

\label{thm:completeness}
Let $\pdata(\x)$ be a (discrete) data distribution. Denote $\cdata(\x; \mathcal{N})$ as the Concrete score of $\pdata(\rvx)$ with neighboring structure $\mathcal{N}$, and $\vc_\theta(\x;  \mathcal{N})$ the Concrete score for a distribution $p_\theta(\x)$ parameterized by 
$\theta \in \Theta$.
When the graph induced by the neighborhood structure $\mathcal{N}$ is connected,
$c_\theta(\x;  \mathcal{N}) = \cdata(\x; \mathcal{N})$ implies that $p_\theta(\x) = \pdata(\x)$ $\forall \x \in \cX$.
\end{restatable}

\begin{proof}[Proof sketch.]
If two nodes $\rvx$ and $\rvx'$ in a neighborhood-induced graph $\cG$ share an edge, then their density ratio $\pdata(\rvx')/\pdata(\rvx)$ can be uniquely identified using $\cdata(\x; \cN)$.
Thus when $\cG$ is connected, the density ratio between any node pairs in $\cG$ is uniquely identified given $\cdata(\x; \cN)$. 
Therefore, knowing the density ratio between any two points uniquely identifies $\pdata(\x)$.
\end{proof}

We provide the full proof in Appendix~\ref{app:proof}. 
We emphasize here that the connectedness of the neighborhood-induced graph (a kind of regularity condition on $\cN$ and $\pdata(\x)$) was crucial for demonstrating the completness of the Concrete score.

\looseness=-1
Note that \Cref{thm:completeness} only requires the neighborhood-induced graph to be connected.
This implies that given a connected graph, we can augment it with additional edges---the new graph will still remain complete, but contain additional information in the augmented neighborhood structure that may help parameter estimation in practice. 
This is a phenomenon that we observe empirically in Section~\ref{2d_multi_path}.

\subsection{Connection to Stein Scores and Existing Score Matching Techniques}
\label{sec:connection}
\looseness=-1
The form of the Concrete score in Eq.~\ref{def:concrete} suggests a natural connection with the Stein score in continuous domains.
In the following proposition, we illustrate the connection between the two when we define
the neighborhood structure to be similar to the grid in Figure~\ref{fig:graph_example}.
In particular, as the distance between neighboring nodes shrinks to zero, our Concrete score converges to the Stein score up to a multiplicative constant.
\begin{proposition}
\label{prop:proposition1}
For $\rvx \in \R^D$, $p(\rvx) \in \gP(\R^D)$, and $\delta > 0$, let the neighborhood structure $\gN_\delta(\rvx) = \{\rvx + \delta \textbf{e}_i\}_{i=1}^{D}$. Then, we have:
\begin{align*}
    \lim_{\delta \to 0} \frac{\vc_p(\x; \cN_\delta)}{\delta} = \nabla_{\x} \log p(\x).
\end{align*}
\end{proposition}

We note that while from this perspective the Concrete score can be seen as a finite-difference approximation to the continuous (Stein) score, CSM is 
different from finite difference score matching (FD-SM) \cite{pang2020efficient}. We introduce a new family of score functions generalizable to discrete domains in the form of $(p(\rvx + \rvv) - p(\rvx)) / p(\rvx)$, where $\rvv$ is the direction from $\rvx$ pointing to its neighbor with length $\delta$.
On the other hand, FD-SM still attempts to match the Stein score, and approximates directional derivatives with finite difference using two neighbors in the form of $\log p(\rvx + \rvv) - \log p(\rvx - \rvv)$. 
Nevertheless, the differences between the two forms are $o(\delta)$ as $\delta \to 0$. %

\subsection{Inference with Concrete Scores}
\label{sec:sampling}
\looseness=-1
To perform inference, we note that there exist several ways to define Markov chain Monte Carlo (MCMC) samplers that only rely on the Concrete score---we highlight the simplest setting of Metropolis-Hastings~\cite{metropolis1953equation,hastings1970monte,chib1995understanding} for clarity of exposition.
We first note that, by definition:
\begin{equation}
\label{eq:mh_ratio}
    \cdata(\x; \cN) + \vone = \bigg[ 
    \frac{\pdata(\x_{n_1})}{\pdata(\x)},
    ...,
    \frac{\pdata(\x_{n_k})}{\pdata(\x)}\bigg]^T
\end{equation}
This implies that we can obtain the density ratios between the neighboring pairs 
$\pdata(\x_{n_1})/\pdata(\x)$
and 
$\pdata(\x)/\pdata(\x_{n_1})$
given the 
Concrete score.
Given a proposal distribution $q$, our sampler will accept the proposed update $\x'$ with acceptance probability 
$A(\rvx' \vert \rvx) = \min \left(1, \frac{\pdata(\rvx')q(\rvx \vert \rvx')}{\pdata(\rvx)q(\rvx' \vert \rvx)} \right)$.
Specifically, we can choose $q$ to 
sample uniformly among $\gN(\x)$ given a particular $\x$. 
When the neighborhood structure is symmetric, this selection leads to a simplified
acceptance probability
$A(\rvx' \vert \rvx) = \min \left(1, \frac{\pdata(\rvx')}{\pdata(\rvx)} \right)$.
We note that when the underlying graph is connected, the chain is aperiodic (due to the presence of a rejection step), irreducible, and positive recurrent (since there is a finite number of discrete states), so the Markov chain is guaranteed to converge in the limit to the model distribution.

{
This means that we can also compute tighter bounds
on log-likelihood estimates obtained by un-normalized probability models trained via CSM. 
In particular, we can compute both a lower bound estimated via Annealed Importance Sampling (AIS) \citep{neal2001annealed} and a conservative upper bound estimated via the Reverse Annealed Importance Sampling Estimator (RAISE) \citep{burda2015accurate}.
This is because both AIS and RAISE allow us to approximate the intractable normalizing constant by constructing a sequence of intermediate distributions between our estimated target distribution and another proposal distribution, and we know that the Concrete score can be repurposed to obtain the necessary density ratios between neighboring pairs of distributions along this sequence.
}

\section{Learning Concrete Scores with Concrete Score Matching}
\subsection{The Concrete Score Matching Objective}
Next, under the assumption that $\cN$ induces a weakly connected graph over the support of $\pdata(\rvx)$, we 
propose a new score matching framework called Concrete Score Matching (CSM) for estimating Concrete scores from data samples.
To do so, we
estimate $\cdata(\x; \cN)$ with a 
score model 
$\vc_\theta (\cdot; \mathcal{N}): \bbR^{D} \rightarrow \bbR^{|\mathcal{N}(\x)|}$, where $|\mathcal{N}(\x)|$ denotes the size of the neighborhood of $\x$.
Our proposed training objective is quite natural, as it measures the average $\ell_2$ difference between our  score model $\vc_\theta(\cdot; \cN)$ and the true score $\cdata(\cdot; \cN)$ similar to (continuous) score matching:
\begin{align}
\label{eq:csm_original}
    \cL_{\mathrm{CSM}} (\theta)= \sum_{\x} \pdata(\x)\norm {\vc_\theta(\x; \mathcal{N})-\cdata(\x; \cN)}_2^2
\end{align}
In the following theorem, we demonstrate that Eq.~\ref{eq:csm_original} is indeed a principled learning objective: the estimated $\theta^*$ is a consistent estimator for the true underlying distribution $\pdata(\rvx)$. 
\begin{restatable}[Consistency]{theorem}{consistency}
\label{thm:consistency}
In the limit of infinite data and infinite model capacity, 
the optimal $\theta^{*}$ that minimizes Eq.~\ref{eq:csm_original} recovers the true Concrete score, or satisfies
$\vc_{\theta^{*}}(\x; \cN)= \cdata(\x; \mathcal{N})$.
\end{restatable}

\Cref{thm:completeness} implies that the estimated (underlying) $p_{\theta^{*}}(\x)$ from $c_{\theta^{*}}(\rvx; \cN)$ satisfies
$p_{\theta^{*}}(\x) = \pdata(\x) \textrm{ } \forall \x \in \cX$.
Next, we note that Eq.~\ref{eq:csm_original} can be simplified to an objective that we can tractably optimize by removing its dependence on the unknown $\cdata(\rvx; \cN)$, similar in spirit to Eq.~\ref{eq:sm_objective}. 
Although the added flexibility of the score model may potentially lead to situations where it does not correspond to a valid probability distribution---similar to how continuous score models do not necessarily form a conservative vector field---this does not appear to hurt CSM's performance empirically.

\begin{restatable}{theorem}{expandedcsm}
\label{thm:tractable}
Optimizing \Cref{eq:csm_original} is equivalent to optimizing:
\begin{align}
    \label{eq:discrete_sm_simplified}
    \cJ_{\mathrm{CSM}}(\theta)&=\underbrace{\sum_{\x}\sum_{i=1}^{|\mathcal{N}(\x)|} \pdata(\x) \bigg( \vc_\theta(\x; \cN)^2_i+2\vc_\theta(\x; \cN)_i \bigg)}_{\cJ_{1}} - 
    \underbrace{\sum_{\x}\sum_{i=1}^{|\mathcal{N}(\x)|}2\pdata(\x_{n_i}) \vc_\theta(\x; \cN)_i}_{\cJ_{2}}
\end{align}
 where $\mathcal{N}(\x) = \{\x_{n_1},...,\x_{n_k}\}$ is the set of neighbors of $\x$.
\end{restatable}
We note that in Eq.~\ref{eq:discrete_sm_simplified}, the objective  $\cJ_{\mathrm{CSM}}(\theta)$ can be approximated using Monte Carlo samples. We elaborate on the empirical training details in \Cref{sec:efficient_training}.

\label{sec:practice}

\subsection{Efficient Training via Monte Carlo}
\label{sec:efficient_training}

In practice, Eq.~\ref{eq:discrete_sm_simplified} is still inefficient; it involves a summation over all $\vert \cN(\rvx) \vert$ neighbors for $\rvx$,
which can be expensive for high-dimensional datasets. 
Therefore, we propose several modifications to the original training objective that we found to be crucial for scaling up and improving our approach.

We begin with the leftmost term $\cJ_{1}$. 
First, we approximate the outer expectation w.r.t. $ \pdata(\x)$ with Monte Carlo samples from the empirical data distribution.
Next, rather than evaluating the objective for all $|\mathcal{N}(\x)|$ neighbors, we sample a neighbor $\rvx_{n_i}$ uniformly at random among $\cN(\rvx_i)$ for every $\rvx_i$ in a mini-batch. We then upweight the output from the model using this sample by $\vert \cN(\rvx) \vert$. We provide the unbiased training algorithm for the leftmost term in \Cref{alg:term1}. We note that is similar in spirit to sliced score matching~\cite{song2020sliced}, an approximation technique to make continuous score matching scalable to higher dimensions.

Next, we turn to the rightmost term $\cJ_{2}$. Similar to $\cJ_1$, we can estimate $\cJ_2$ using an unbiased estimator based on samples as demonstrated in \Cref{alg:term2}.
We define $\gN^{-1}(\x') = \{(\x, i) | \gN(\x)_i = \x'\}$ as the
``reverse neighborhood'' set,
where an element $(\x, i) \in \gN^{-1}(\x')$ indicates that $\x'$ is the $i$-th neighbor of $\x$. There could be multiple $\x$ with the same $i$ in $\gN^{-1}(\x')$ as in the case of the star graph. Computing and storing $\gN^{-1}$ as a hash table mapping $\x$ to the set of $\gN^{-1}(\x)$ has a time and space complexity of at most $O(\sum_\x |\gN(\x)|)$ (\textit{i.e.}, the number of edges). For special structures (\textit{e.g.}, grid), we can design specific algorithms that bypass this process. We provide more details in \Cref{app:experiment}. 

\begin{minipage}{0.48\textwidth}
\begin{algorithm}[H]
\caption{An unbiased estimator of $\gJ_1$}
\label{alg:term1}
\textbf{Input:} $\pdata$, $\mathcal{N}$, $\vc_\theta$ (model). \\
1. Sample $\x\sim\pdata(\rvx)$.\\
2. Sample $i \sim \mathrm{Uniform}\{1, \ldots, |\gN(\x)|\}$.\\
\textbf{Output:}  $|\gN(\x)| \cdot \left( \vc_\theta(\x; \mathcal{N})^2_{i}+2\vc_\theta(\x; \mathcal{N})_{i} \right)$.
\end{algorithm}
\end{minipage}
\hfill
\begin{minipage}{0.48\textwidth}
\begin{algorithm}[H]
\caption{An unbiased estimator of $\gJ_2$}\label{alg:term2}
\textbf{Input:} $\pdata$, $\mathcal{N}$, $\vc_\theta$ (model). \\
1. Sample $\x' \sim\pdata(\rvx)$.\\
2. Sample $(\x, i) \sim \mathrm{Uniform}(\gN^{-1}(\x'))$. \\
\textbf{Output:}  $2 \cdot |\gN^{-1}(\x')| \cdot \vc_\theta(\x; \mathcal{N})_{i}$.
\end{algorithm}
\end{minipage}

We leverage the grid structure in our experiments, where the neighbors of $\rvx$ are defined to be $\cN(\rvx) = \{\rvx \pm \rve_d\}_{d=1}^D$.
\Cref{alg:term1} and \Cref{alg:term2} allow us to  efficiently train the model by sampling a dimension $d \in [D]$ and flipping the bit for the $d$th entry of $\rvx$ 
for each $\rvx_n$.

\subsection{Denoising Concrete Score Matching}
\label{sec:denoising}

Finally, we propose another training objective to approximate Eq.~\ref{eq:discrete_sm_simplified}, with a focus on computational efficiency.
Similar to denoising score matching (DSM) for (Stein) score estimation \cite{vincent2011connection}, we derive a denoising counterpart for Concrete score estimation called ``Denoising Concrete Score Matching'' (D-CSM).
Specifically, given a discrete data distribution $\pdata(\x)$ and a discrete noise distribution $\tilde{q}(\tx | \x)$, 
we define the perturbed data distribution $\tilde{p}(\tx) = \sum_{\x} \pdata(\x) \tilde q( \tx |\x)$, and the posterior $q(\x | \tx) = \frac{\pdata(\x) \tilde{q}( \tx | \x)}{\tilde p(\tx)}$.
Then, we can show that the Concrete score of the perturbed data distribution $\tilde{p}(\tx)$ can be obtained via $\vc_{\tilde p(\tx)}(\tx; \mathcal{N}) = \sum_{\x} \vc_{\tilde{q}(\tx| \x)}( \tx; \mathcal{N}) {q(\x | \tx)}$. 
This property allows us to obtain the D-CSM objective, as shown in the following theorem:
\begin{restatable}[Denoising Concrete Score Matching]{theorem}{propertytwo}
The objective:
\begin{align}
\label{eq:denoising_csm}
    \cJ_{\mathrm{D-CSM}}(\theta)=\sum_{\x,\tx} \pdata(\x) \tilde{q}(\tx | \x) \norm{\vc_\theta(\tx; \mathcal{N}) - \vc_{q( \tx | \x)}(\tx; \mathcal{N})}_2^2 
\end{align}
is minimized when $\vc_\theta(\tx; \mathcal{N}) = \vc_{p(\tx)}(\tx; \mathcal{N} )$. 
\end{restatable}

\looseness=-1
We draw a direct analogy to DSM in order to provide additional insight into D-CSM. Recall that in DSM, there exists a noise distribution $q(\tx|\x)$ (\ie $\mathcal{N}(\tx|\x,\sigma^2 I)$) such that the magnitude of Stein scores captures the distance between the perturbed $\tx$ and its clean counterpart $\x$.
D-CSM enjoys a similar interpretation.
In particular, consider an input space $\x \in \cX \subseteq \mathbb{Z}^{D}$ with the neighborhood structure $\mathcal{N}(\x)=\{\x+\rvv | \rvv \in \{-1, 0, 1\}^{D}\}$. If we define $q(\tx|\x)=\frac{1}{2^{D+k}}$ where $k$ is the Manhattan distance between $\tx$ and $\x$, then the Concrete score $\vc_{\tilde{p}(\tx)}(\tx;\mathcal{N})$ captures how close $\rvx$ is to $\tx$
as measured by the Manhattan distance. Therefore, D-CSM can also be understood through the lens of a denoiser.
{We provide additional experimental results exploring the performance of D-CSM in Appendix~\ref{app:dcsm_addtl}.}

\section{Experimental Results}
\label{sec:experiments}
In this section, 
we evaluate the performance of CSM on synthetic, tabular, and discrete image datasets on a variety of sampling and density estimation tasks. We provide additional details on specific experimental settings and hyperparameter configurations in Appendix~\ref{app:experiment}.

\subsection{Baselines}
In our experiments, we compare CSM to two relevant baselines for modeling discrete data: Ratio Matching and Discrete Score Matching with Marginalization (\texttt{Discrete Marginalization}). For conciseness, we provide additional details and derivations for their training objectives in \Cref{app:addtl_works}.

\textbf{Ratio Matching.} Similar to score matching, Ratio Matching \cite{lyu2012interpretation} leverages the fact that ratios of probabilities are independent of the intractable normalizing constant (due to cancellation).
The method then seeks to match the ground truth density ratios $\frac{\pdata(\rvx)}{\pdata(\rvx_{-d})} = \frac{q_\theta(\rvx)}{q_\theta(\rvx_{-d})}$ where $\rvx_{-d}$ denotes the vector $\rvx$ with the $d$th entry bit-flipped (e.g. from 0 to 1).

\textbf{Discrete Score Matching with Marginalization (Discrete Marginalization).} The Discrete Marginalization baseline is another way of estimating discrete probability distributions with score matching as in \cite{lyu2012interpretation}. However, we note that this approach is difficult to scale because it requires us to marginalize over 
the data dimension, which may also cause instabilities during training. 

\subsection{1-D Discrete Data}
\looseness=-1
\label{sec:1d_data}
In this experiment, we consider a 16-category 1-D data distribution as shown in \Cref{fig:1d_examples}.
We parameterize our Concrete score model $\vc_\theta(\rvx, \cN)$ using a shallow MLP model with Tanh activations, and use the Cycle neighborhood structure during training.
We demonstrate that the learned Concrete score model can faithfully capture the true data distribution when trained with CSM.
As shown on the left side of Figure~\ref{fig:1d_examples}, we find that CSM generates samples (\textcolor{cyan}{blue}) using MH that almost perfectly matches the samples drawn from $\pdata(\rvx)$ (\textcolor{green}{green}).

We also observe in Appendix~\ref{app:langevin_noise_perturbation} that the Concrete score on $\pdata(\x)$ can recover the Stein score of a data distribution perturbed with triangular noise. 
Even though the Concrete score model is trained on discrete data (\textcolor{green}{green}), this connection allows us to sample with Langevin dynamics using the recovered Stein score. On the right side of \Cref{fig:1d_examples}, we show how the samples generated using Langevin dynamics (\textcolor{orange}{smoothed orange}) closely match the triangular noise-perturbed data distribution (\textcolor{gray}{gray}). 
Then, we leverage a closed-form denoising formula for the perturbed distribution (Appendix~\ref{app:langevin_noise_perturbation}) to recover the clean data distribution (\textcolor{green}{green}) from samples obtained via Langevin dynamics (\textcolor{orange}{smoothed orange}). The resulting denoised samples (\textcolor{orange}{orange}) are labeled as \texttt{Denoised} in \Cref{fig:1d_examples}. We provide more details and discussion in \Cref{app:langevin_noise_perturbation}.

\begin{figure}[ht]
    \centering
    \includegraphics[width=\linewidth]{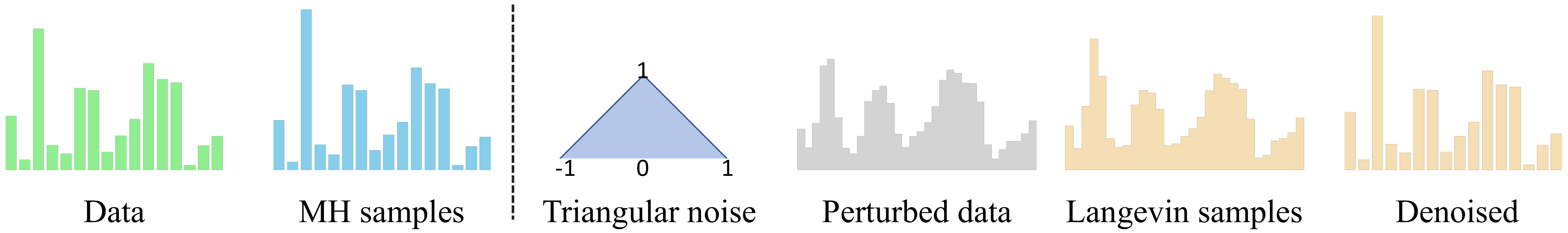}
    \caption{Sampling results from a toy 1-D discrete dataset. CSM recovers the true data distribution $\pdata(\rvx)$ using Metropolis-Hastings (left).
    The Concrete score can also be used to recover the Stein score of the triangular noise-perturbed data distribution, allowing for sampling with Langevin dynamics. Such samples can be denoised to recover the original (clean) data distribution (right).
    }
    \label{fig:1d_examples}
    \vspace{-15pt}
\end{figure}

\subsection{Sampling with Toy 2-D Multi-Class Datasets}
\label{2d_multi_path}
Next, we consider three 2-D toy benchmark datasets with multiple modes and discontinuities as commonly used in the density estimation literature \cite{grathwohl2018ffjord,de2020block}. 
We quantize the data into $91\times91$ bins to obtain the discrete  training  data for the experiment (see \Cref{fig:2d_samples}).
We compare our method against the two baselines: (1) Ratio Matching~\cite{hyvarinen2007connections}; and (2) 
Discrete Marginalization~\cite{lyu2012interpretation}. In particular, we found that the original equations in \cite{lyu2012interpretation} were incorrect, and provide results using the correct training objectives (denoted as \texttt{Ratio-fixed} and \texttt{Marginal-fixed}) in \Cref{fig:2d_samples}. We provide a step-by-step derivation addressing this issue with the correct expressions for the training objectives in \Cref{app:addtl_works}.

For a fair comparison, we use the same model architecture and training configurations across all methods. 
We use the grid neighborhood structure for training CSM. 
As shown in Figure~\ref{fig:2d_samples}, the samples from CSM best capture the shape of the underlying data distributions across all three datasets (leftmost column). 
Baselines implemented using the original equations in \cite{lyu2012interpretation} indeed demonstrate poor performance on the density estimation task (\texttt{Ratio} and \texttt{Marginal} columns) in Figure~\ref{fig:2d_samples}.

\begin{figure}[ht]
     \centering
     \begin{subfigure}[b]{0.15\textwidth}
         \centering
         \includegraphics[width=\textwidth]{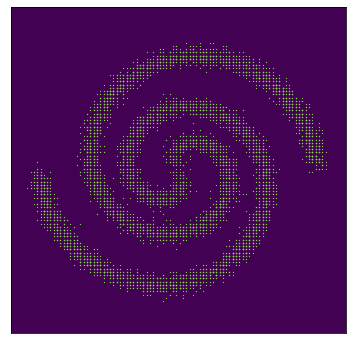}
     \end{subfigure}
    \begin{subfigure}[b]{0.15\textwidth}
         \centering
         \includegraphics[width=\textwidth]{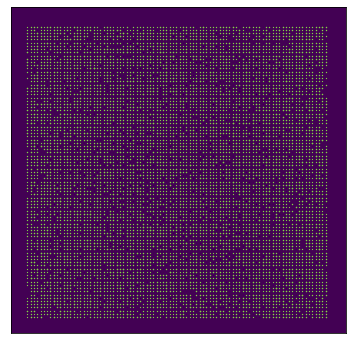}
     \end{subfigure}
     \begin{subfigure}[b]{0.15\textwidth}
         \centering
         \includegraphics[width=\textwidth]{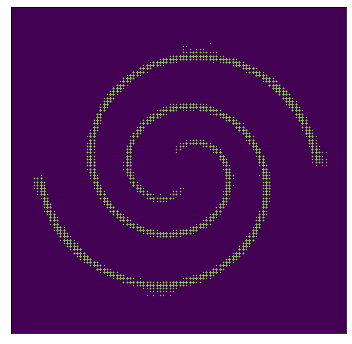}
     \end{subfigure}
      \begin{subfigure}[b]{0.15\textwidth}
         \centering
         \includegraphics[width=\textwidth]{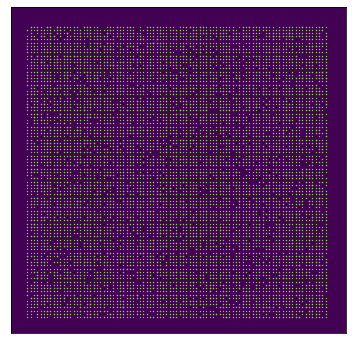}
     \end{subfigure}
     \begin{subfigure}[b]{0.15\textwidth}
         \centering
         \includegraphics[width=\textwidth]{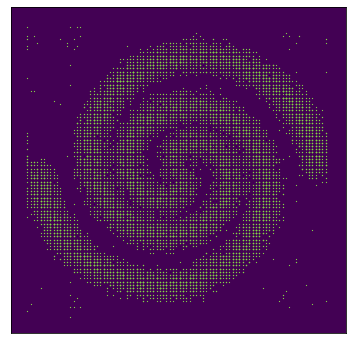}
     \end{subfigure}
     \begin{subfigure}[b]{0.15\textwidth}
         \centering
         \includegraphics[width=\textwidth]{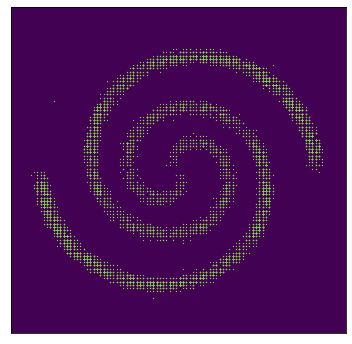}
     \end{subfigure}
    
    \begin{subfigure}[b]{0.15\textwidth}
         \centering
         \includegraphics[width=\textwidth]{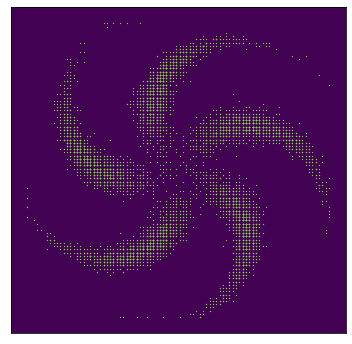}
     \end{subfigure}
     \begin{subfigure}[b]{0.15\textwidth}
         \centering
         \includegraphics[width=\textwidth]{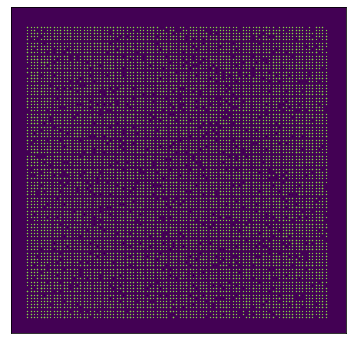}
     \end{subfigure}
     \begin{subfigure}[b]{0.15\textwidth}
         \centering
         \includegraphics[width=\textwidth]{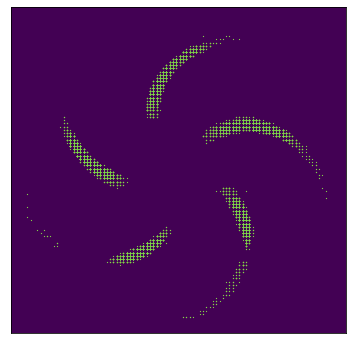}
     \end{subfigure}
      \begin{subfigure}[b]{0.15\textwidth}
         \centering
         \includegraphics[width=\textwidth]{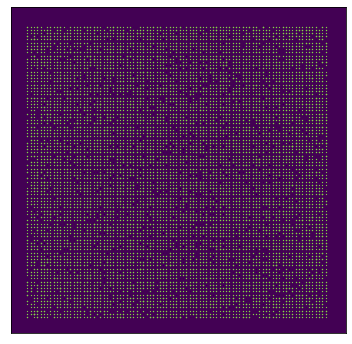}
     \end{subfigure}
     \begin{subfigure}[b]{0.15\textwidth}
         \centering
         \includegraphics[width=\textwidth]{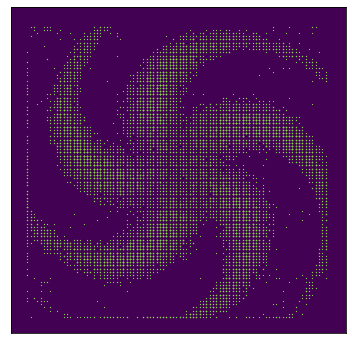}
     \end{subfigure}
     \begin{subfigure}[b]{0.15\textwidth}
         \centering
         \includegraphics[width=\textwidth]{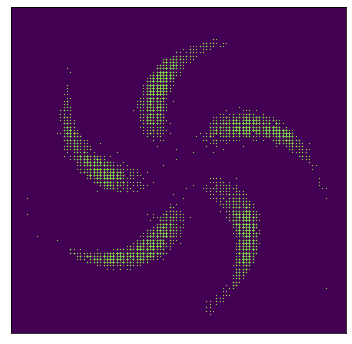}
     \end{subfigure}

     \begin{subfigure}[b]{0.15\textwidth}
         \centering
         \includegraphics[width=\textwidth]{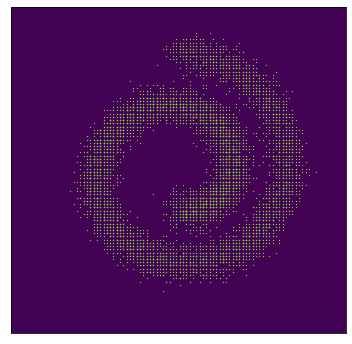}
         \caption*{Data}
     \end{subfigure}
     \begin{subfigure}[b]{0.15\textwidth}
         \centering
         \includegraphics[width=\textwidth]{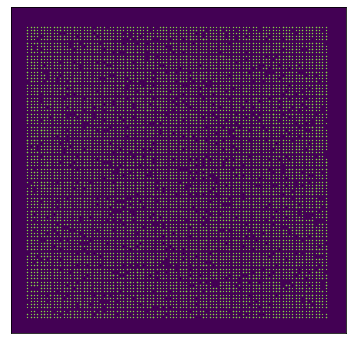}
         \caption*{Ratio}
     \end{subfigure}
     \begin{subfigure}[b]{0.15\textwidth}
         \centering
         \includegraphics[width=\textwidth]{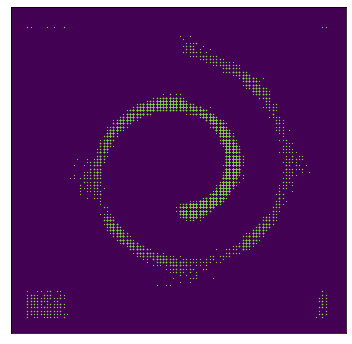}
         \caption*{Ratio-fixed}
     \end{subfigure}
      \begin{subfigure}[b]{0.15\textwidth}
         \centering
         \includegraphics[width=\textwidth]{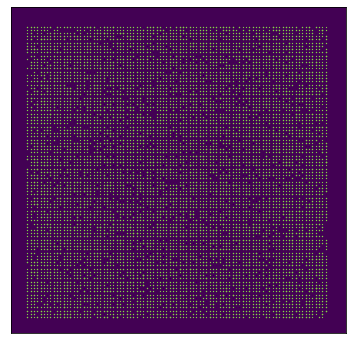}
         \caption*{Marginal}
     \end{subfigure}
     \begin{subfigure}[b]{0.15\textwidth}
         \centering
         \includegraphics[width=\textwidth]{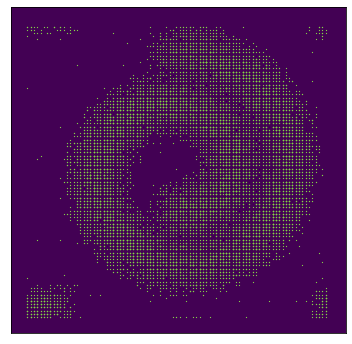}
         \caption*{Marginal-fixed}
     \end{subfigure}
     \begin{subfigure}[b]{0.15\textwidth}
         \centering
         \includegraphics[width=\textwidth]{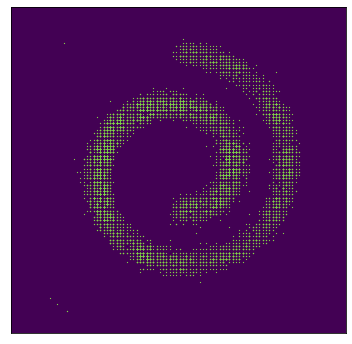}
         \caption*{CSM (Ours)}
     \end{subfigure}
     
     \caption{Sampling results on toy 2-D benchmark datasets. We find that CSM produces the highest quality samples across all 3 datasets relative to all baselines.}
      \label{fig:2d_samples}
      \vspace{-5pt}
\end{figure}

\subsection{Likelihood Evaluation on Discrete Tabular Data}
\label{exp:tabular}
\looseness=-1
We then evaluate the performance of CSM on density estimation tasks for a wide range of tabular (discrete) datasets drawn from both the Twenty Datasets \cite{van2012markov} and the Amazon Baby Registries benchmarks \cite{gillenwater2014expectation}. 
In order to evaluate likelihoods, we directly parameterize the density with a discrete autoregressive model (MADE~\cite{germain2015made}), but train the model using the baseline approaches and CSM. For a fair comparison, we use the same model architecture and experimental configurations across all methods. 
Similar to our previous experiments, we use the grid neighborhood structure for training CSM.
As shown in Table~\ref{tab:tabular_likelihood}, our approach (\texttt{CSM}) demonstrates favorable performance relative to the \texttt{Ratio Matching} and \texttt{Discrete Marginalization} baselines.
We provide {additional experimental results in Appendix~\ref{app:addtl_ll_table} and} more experimental details in \Cref{app:experiment}.

\begin{table}
\vspace{-10pt}
\newcommand{\mround}[1]{\round{#1}{2}}
    \newcommand{\nastar}{\multicolumn{1}{c}{\hspace{15pt}--~*}}
    \newcommand{\na}{\multicolumn{1}{c}{--}}
    \centering %
    {
        \centering %

        \begin{adjustbox}{max width=0.8\linewidth}
        \begin{tabular}{l @{\extracolsep{5pt}}  c c c c @{}}
        \\
        \toprule
        {Datasets}
        & Ratio Matching $(\uparrow)$
        & Discrete Marginalization $(\uparrow)$
        & CSM (Ours) $(\uparrow)$
        \\ 
        \midrule
        NLTCS %
        & -6.15 %
        & -6.21 %
        & \textbf{-6.13} %
        \\
        Plants%
        & -15.44 %
        & -19.03 %
        & \textbf{-14.02} %
        \\
        Jester %
        & -56.49 %
        & -57.06 %
        & \textbf{-54.91} %
        
        \\
        Amazon Diaper
        & \textbf{-10.69} %
        & -42.52 %
        & -11.13 %
        \\
        Amazon Feeding
        & \textbf{-12.09} %
        & -35.96 %
        & -12.65 %
        \\
        Amazon Gifts
        & -4.57 %
        & -4.28 %
        & \textbf{-4.22} %
        \\
        Amazon Media
        & \textbf{-10.22} %
        & -13.77 %
        & -10.30 %
        \\
        Amazon Toys
        & -9.83 %
        & -16.34 %
        & \textbf{-9.30} \\
    \bottomrule
    \end{tabular}
    \end{adjustbox}
}
\caption{
Log-likelihood comparisons on discrete tabular datasets. Higher is better.
We find that CSM demonstrates good performance, almost always outperforming or performing comparably relative to the \texttt{Ratio Matching} and \texttt{Discrete Marginalization} baselines.
}
\label{tab:tabular_likelihood}
\vspace{-15pt}
\end{table}

\subsection{Generative Modeling with High-dimensional Images}
For our final experiment, we demonstrate that we can scale CSM to achieve good performance on complex, high-dimensional binarized image datsets. We experiment with the MNIST \cite{lecun1998mnist} dataset, which has 784 dimensions.
We directly parameterize the Concrete score function with a U-Net~\cite{ronneberger2015u}, and train the model using the CSM based on a grid neighborhood structure. Similar to \cite{song2019generative}, we perturb the image with different levels of discrete (Categorical) noise, and train the models at different noise levels with annealing. 
After the Concrete score model is trained, we sample from the model using Metropolis-Hastings as discussed in \Cref{sec:sampling} and follow a similar sample initialization process as ~\cite{song2019generative}. We provide more details in \Cref{app:experiment}.
We present the samples from our model in ~\Cref{fig:mnist_samples}. We observe that  our CSM approach with Metropolis-Hastings is able to generate samples that look similar to the digit examples in the training set (Figure~\ref{fig:train}). 

\begin{figure}[ht]
     \centering
     \begin{subfigure}[h]{0.3\textwidth}
         \centering
         \includegraphics[width=\textwidth]{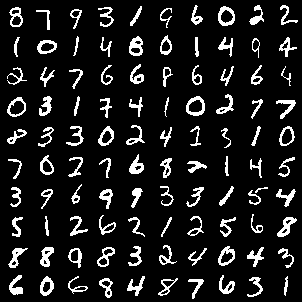}
         \caption{Binary training images}
         \label{fig:train}
     \end{subfigure}
    \hspace{0.05\textwidth}
     \begin{subfigure}[h]{0.3\linewidth}
         \centering
         \includegraphics[width=\textwidth]{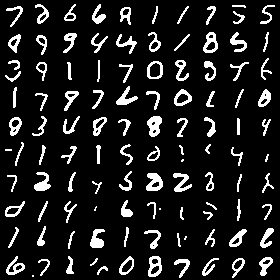}
         \caption{MNIST samples (MH)}
         \label{fig:mnist_samples}
     \end{subfigure}
         \label{fig:fashion_mnist_samples}
        \caption{Image samples from CSM on the binarized MNIST dataset.
        We observe that CSM is able generate high-quality samples on MNIST using Metropolis-Hastings.}
        \label{fig:image_samples}
        \vspace{-5pt}
\end{figure}

\section{Related Work}
\textbf{Score Matching.}
\looseness=-1
Our work builds on the body of literature on score matching \cite{hyvarinen2005estimation,vincent2011connection} and score-based generative modeling \cite{song2020sliced,song2019generative,song2020score}. Notably, we build upon the generalization of score matching to discrete data \cite{hyvarinen2007connections,hyvarinen2007some,lyu2012interpretation}. 
Our score function can be viewed through the lens of generalized score matching as in \cite{lyu2012interpretation}, where the Concrete score serves as a particular instantiation of the linear operator in their framework.
However, we provide a novel perspective in terms of constructing the Concrete score via surrogate gradient information over a structured discrete space, and provide efficient training objectives for CSM.
Our work is also closely related to score matching with finite difference (FD-SM) \cite{pang2020efficient} as discussed in Section~\ref{sec:connection}. However, rather than approximating the gradient with finite differences for continuous data, we leverage the first-order forward difference to construct our Concrete score function in discrete domains.
{In addition, CSM bears similarities to a concurrent work on scaling up Ratio Matching \cite{lyu2012interpretation} for discrete energy-based models (EBMs) \cite{liu2021gradient} called RMwGGIS---our approach can be viewed as a different way to train discrete EBMs (via the Concrete score).
}

\textbf{Generative Modeling for Discrete Data.}
CSM also provides another way to train generative models for discrete data. 
A related work is \cite{niu2020permutation}, which trains a score-based model on distributions over graphs. 
However, they perturb the adjacency matrices with Gaussian noise and use the continuous variant of DSM.
Although there exist several model families for learning binary and Categorical probability distributions such as normalizing flows \cite{ziegler2019latent,tran2019discrete,hoogeboom2021argmax}, Sum-Product Networks (SPNs) \cite{poon2011sum,Molina2019SPFlow}, {denoising diffusion probabilistic models \cite{austin2021structured}}, {discrete (latent) \cite{bao2020bi,bao2021variational} EBMs \cite{liu2021gradient}}, and Generative Flow Networks (GFNs) \cite{zhang2022generative}, there does not yet exist a discrete score-based model that can scale to high-dimensional discrete datasets.
Finally, our approach bears similarities with Gibbs with Gradients (GWG) \cite{grathwohl2021oops}.
In particular, GWG augments existing MCMC samplers with gradient information by leveraging local structure to estimate likelihood ratios for transitioning to the next state.
However, GWG utilizes gradients from the probability mass functions of the underlying discrete distributions, in contrast to the way in which we construct the Concrete score function.

\section{Conclusion}
\label{sec:conc}
We introduced Concrete Score Matching (CSM), a novel framework for learning discrete probability distributions via score matching. 
We proposed to leverage particular kinds of structural information in the data to construct \emph{surrogate gradient information} about the discrete space, and used this to define a valid score function.
We also introduced several modifications to the original training objective that allowed CSM to scale gracefully to high-dimensional datasets, and demonstrated that CSM performs well on a variety of sampling and density estimation tasks.

However, this work is not without limitations. Since CSM depends on the neighborhood structure, certain types of graphs may work better for score matching in practice than others.
Additionally, our efficient training objectives may suffer from high variance when scaling to high dimensions for particular kinds of neighborhood-induced graphs, such as the Star graph.
For future work, although we fixed the number of neighbors $K$ for each $\rvx$ in our experiments, it would be interesting to adaptively determine the optimal number of neighbors to use for each $\rvx$ during training.
We also believe that CSM can be generalized to leverage more complex neighborhood structures than those we explored in the paper.
Additionally, empirically investigating the performance of D-CSM with Langevin dynamics relative to CSM would be interesting future work.

\paragraph{Broader Impact.}
\label{sec:impact}
This work introduces a novel score function---the Concrete score---that is defined over discrete spaces, as well as a corresponding score matching (CSM) framework that can scale to high-dimensional discrete datasets.
This leads to empirical performance improvements over a range of density estimation and sampling tasks, and does not have a direct consequence on societal issues. 
However, we note that CSM could serve as the basis for developing more powerful generative models of structured, discrete data. Although this could lead to tangible benefits (e.g. improved generative modeling of text data), we should be mindful to take the usual precautions required for generative modeling research (e.g. against the development of deepfakes). 

\section*{Acknowledgements and Funding Disclosure}
We thank the anonymous reviewers for insightful discussions and feedback.
KC is supported by the Qualcomm Innovation Fellowship and the Two Sigma Diversity PhD Fellowship.
This research was supported by NSF (\#1651565), AFOSR (FA95501910024), ARO (W911NF-21-1-0125), ONR, DOE, CZ Biohub, and Sloan Fellowship.

\appendix

\bibliography{references}
\newpage
\section*{Checklist}

\begin{enumerate}

\item For all authors...
\begin{enumerate}
  \item Do the main claims made in the abstract and introduction accurately reflect the paper's contributions and scope?
    \answerYes{}
  \item Did you describe the limitations of your work?
    \answerYes{Section~\ref{sec:conc}}
  \item Did you discuss any potential negative societal impacts of your work?
    \answerYes{Section~\ref{sec:impact}}
  \item Have you read the ethics review guidelines and ensured that your paper conforms to them?
    \answerYes{}
\end{enumerate}

\item If you are including theoretical results...
\begin{enumerate}
  \item Did you state the full set of assumptions of all theoretical results?
    \answerYes{Sections~\ref{sec:method}-\ref{sec:practice}}
        \item Did you include complete proofs of all theoretical results?
    \answerYes{Appendix~\ref{app:proof}}
\end{enumerate}

\item If you ran experiments...
\begin{enumerate}
  \item Did you include the code, data, and instructions needed to reproduce the main experimental results (either in the supplemental material or as a URL)?
    \answerNo{We will release the code publicly upon publication.}
  \item Did you specify all the training details (e.g., data splits, hyperparameters, how they were chosen)?
    \answerYes{Section~\ref{sec:experiments} and Appendix~\ref{app:experiment}}
        \item Did you report error bars (e.g., with respect to the random seed after running experiments multiple times)?
    \answerNo{}
        \item Did you include the total amount of compute and the type of resources used (e.g., type of GPUs, internal cluster, or cloud provider)?
    \answerYes{Appendix~\ref{app:experiment}}
\end{enumerate}

\item If you are using existing assets (e.g., code, data, models) or curating/releasing new assets...
\begin{enumerate}
  \item If your work uses existing assets, did you cite the creators?
    \answerYes{Section~\ref{sec:experiments}}
  \item Did you mention the license of the assets?
    \answerNo{We only used publicly available datasets for our experiments.}
  \item Did you include any new assets either in the supplemental material or as a URL?
    \answerNA{}
  \item Did you discuss whether and how consent was obtained from people whose data you're using/curating?
    \answerNA{}
  \item Did you discuss whether the data you are using/curating contains personally identifiable information or offensive content?
    \answerNA{}
\end{enumerate}

\item If you used crowdsourcing or conducted research with human subjects...
\begin{enumerate}
  \item Did you include the full text of instructions given to participants and screenshots, if applicable?
    \answerNA{}
  \item Did you describe any potential participant risks, with links to Institutional Review Board (IRB) approvals, if applicable?
    \answerNA{}
  \item Did you include the estimated hourly wage paid to participants and the total amount spent on participant compensation?
    \answerNA{}
\end{enumerate}

\end{enumerate}

\bibliographystyle{unsrt}
\newpage
\section*{Appendix}

\section{Proofs of Theoretical Results}
\label{app:proof}
In this section, we provide proofs and additional discussions of the theoretical results.

\subsection{Concrete Score Matching}

\completeness*
\begin{proof}
If two nodes $\rvx$ and $\rvx'$ in a neighborhood-induced graph $\cG$ share an edge, then their density ratio $\pdata(\rvx')/\pdata(\rvx)$ and $\pdata(\rvx)/\pdata(\rvx')$ can be uniquely identified using $\cdata(\x; \cN)$ based on the definition:
\begin{equation}
    \cdata(\x; \cN) + \vone = \bigg[ 
    \frac{\pdata(\x_{n_1})}{\pdata(\x)},
    ...,
    \frac{\pdata(\x_{n_k})}{\pdata(\x)}\bigg]^T.
\end{equation}

For simplicity, we name the elements in $\mathcal{X}$: $\{\x_1,\x_2,..., \x_N\}$.
Note that when $\cG$ is a weakly connected graph, any node pair $(\x_1, \x_n)$ is connected via a path in the graph. 
Denoting the path between these two nodes as $\x_1\to...\to \x_n$, we can obtain the density ratio for any neighboring pairs on the path using the definition of $\cdata(\x; \cN)$. Therefore the density ratio $\pdata(\rvx_n)/\pdata(\rvx_{1})$ can be computed using the products of the density ratios for all neighboring data pairs on the path. 
This implies that when $\cG$ is weakly connected, we can uniquely recover the density ratios
\{$\pdata(\rvx_1)/\pdata(\rvx_{1}),
\pdata(\rvx_2)/\pdata(\rvx_{1}), \ldots,  \pdata(\rvx_N)/\pdata(\rvx_{1})\}$, which uniquely determine $\pdata(\x)$.
Therefore, knowing the density ratio between any two points uniquely identifies $\pdata(\x)$.

\end{proof}

\consistency*
\begin{proof}
It is easy to see the optimal $\theta^{*}$ that minimizes the following equation
\begin{align}
    \cL_{\mathrm{CSM}} (\theta)= \sum_{\x} \pdata(\x)\norm {\vc_\theta(\x; \mathcal{N})-\cdata(\x; \cN)}_2^2
\end{align}
satisfies $\vc_{\theta^{*}}(\x; \cN)= \cdata(\x; \mathcal{N})$ almost everywhere.

\end{proof}

\expandedcsm*
\begin{proof}
Recall that by definition
\begin{equation}
    \cdata(\x ;\mathcal{N}) \triangleq \bigg[ 
    \frac{\pdata(\x_{n_1})-\pdata(\x)}{\pdata(\x)},
    ...,
    \frac{\pdata(\x_{n_k})-\pdata(\x)}{\pdata(\x)}\bigg]^T.
\end{equation}
\begin{align*}
&\argmin_{\theta} \cL_{\mathrm{CSM}} (\theta)\\
&=\argmin_{\theta} \sum_{\x} \pdata(\x)\norm {\vc_\theta(\x; \mathcal{N})-\cdata(\x; \cN)}_2^2\\
&=\argmin_{\theta} \sum_{\x}\pdata(\x)\bigg[\norm {\cdata(\x; \cN)}_2^2-2\vc_\theta(\x; \mathcal{N})^T\cdata(\x; \cN)+\norm {\vc_\theta(\x; \mathcal{N})}_2^2\bigg]\\
&=\argmin_{\theta} \sum_{\x}\pdata(\x)\bigg[\norm {\vc_\theta(\x; \mathcal{N})}_2^2-2\vc_\theta(\x; \mathcal{N})^T\cdata(\x; \cN)\bigg]\\
&=\argmin_{\theta} \underbrace{\sum_{\x}\sum_{i=1}^{|\mathcal{N}(\x)|} \pdata(\x) \bigg( \vc_\theta(\x; \cN)^2_i+2\vc_\theta(\x; \cN)_i \bigg)}_{\cJ_{1}} - 
    \underbrace{\sum_{\x}\sum_{i=1}^{|\mathcal{N}(\x)|}2\pdata(\x_{n_i}) \vc_\theta(\x; \cN)_i}_{\cJ_{2}}\\
    &=\argmin_{\theta} \cJ_{\mathrm{CSM}}(\theta)
\end{align*}

\end{proof}
\subsection{Denoising Concrete Score Matching}

\begin{restatable}{property}{propertyone}
\label{prop:1}
$\vc_{\tilde p(\tx)}(\tx; \mathcal{N}) = \sum_{\x} \vc_{\tilde{q}(\tx| \x)}( \tx; \mathcal{N}) {q(\x | \tx)}$
\end{restatable}

\begin{proof}
For simplicity, given a distribution $p(\x)$, we define the operator $L[p(\x)]=\bigg[ {p}(\x_{n_1})-{p}(\x),..., {p}(\x_{n_k})-{p}(\x)\bigg]^T$, where $\cN(\rvx) = \{\x_{n_1},\ldots,\x_{n_k}\}$ are the neighbors of $\x$. 
Recall that by definition, $\tilde{p}(\tx) = \sum_{\x} \pdata(\x) \tilde q( \tx |\x)$ and the posterior $q(\x | \tx) = \frac{\pdata(\x) \tilde{q}( \tx | \x)}{\tilde p(\tx)}$.
We have
\begin{align*}
    c_{\tilde{p}(\tx)}(\tx; \mathcal{N})
    &=\frac{L[\tilde{p}(\tx)]}{\tilde{p}(\tx)}\\
    &= \frac{L[\sum_{\x} \pdata(\x) \tilde{q}( \tx | \x)]}{\tilde{p}(\tx)} \\
    &= \frac{\sum_{\x} \pdata(\x) L[ \tilde{q}( \tx | \x)]}{\tilde{p}(\tx)}\;\text{(by linearity)} \\
    &=\sum_{\x} \frac{ L[ \tilde{q}( \tx | \x)]}{\tilde{q}( \tx | \x)}\frac{\pdata(\x)\tilde{q}( \tx | \x)}{\tilde{p}(\tx)} \\
    &=\sum_{\x} c_{\tilde{q}( \tx | \x)}(\tx; \mathcal{N}) {q}(\x | \tx) 
\end{align*}

\end{proof}

\propertytwo*

\begin{proof}
The model $\vc_\theta$ that minimizes the least squares 
\begin{align}
    \theta^{*} &= \argmin_{\theta} \sum_{\x, \tx} \pdata(\x) \tilde{q}(\tx | \x) || \vc_{\theta}(\tx;\mathcal{N}) - c_{\tilde{q}( \x | \tx)}(\tx;\mathcal{N}) ||^2 
\end{align}
satisfies $\vc_{\theta}(\tx;\mathcal{N})=\sum_{\x}c_{\tilde{q}( \tx | \x)}(\tx;\mathcal{N}){q}( \x | \tx) =c_{\tilde{p}(\tx)}(\tx; \mathcal{N})$ using \Cref{prop:1}.
\end{proof}

\subsection{Connection to distribution perturbed with triangular noise}
\label{app:langevin_noise_perturbation}

As discussed in \Cref{sec:1d_data}, our approach can be used to denoise a given data distribution perturbed with triangular noise. 
For simplicity, we assume $\x \in \mathcal{X} \subseteq \mathbb{Z}^{D}$.
The argument proceeds in a similar manner to how denoising score matching can be used to compute the expected posterior $\mathbb{E}_{\x}[p(\x|\tilde \x)]$, where $\tilde \x$ is the perturbed data.

\begin{definition}[Triangular noise]
We define the PDF of the $D$-dimensional triangular distribution with lower limit -1, upper limit 1, and mode 0 as the following:
\begin{equation}
\label{def:triangular_noise}
\mathcal{T}_D(\x)=
\begin{cases}
0 & \text{$\x<-1$}\\
\x+1 & \text{$-1\le\x\le 0$}\\
1-\x  & \text{$0<\x\le 1$} \\
0 & \text{$\x>1$}
\end{cases}
\end{equation}
\end{definition}
\begin{lemma}
\label{lemma:triangular_noise}
Given a $D$-dimensional discrete distribution  $\pdata(\x)$,
let $\tilde{p}(\tx)$ be the distribution of $\pdata(\x)$ when perturbed with triangular noise as defined in \Cref{def:triangular_noise}.
\begin{align}
    \tilde{p}(\tx) \triangleq \sum_{\x} \pdata(\x) \mathcal{T}_D(\tx-\x).
\end{align}
Then for any $\x, \y \sim \pdata(\x)$ we have
$\frac{\tilde{p}(\x)}{\tilde{p}(\y)}=\frac{{\pdata}(\x)}{\pdata(\y)}.$
\end{lemma}
\begin{proof}
It is easy to see that for any $\x,\y \in\mathcal{X}\subseteq \mathbb{Z}^{D}$, we have 
\begin{equation}
    \tilde{p}(\x) = \pdata(\x) \mathcal{T}_D(\textbf{0}), \; \tilde{p}(\y) = \pdata(\y) \mathcal{T}_D(\textbf{0})
\end{equation}
because the width of the triangular noise is less than one.
Thus $\frac{\tilde{p}(\x)}{\tilde{p}(\y)}=\frac{{\pdata}(\x)}{\pdata(\y)}.$

\end{proof}

Denote $\bar{\x}_d$ the integer such that the $d$-th index of $\tx$, denoted $\tx_d$, satisfies $\tilde \x_d \in [\bar{\x}_d, \bar{\x}_d+1)$.
Similar to \Cref{sec:denoising}, in the discrete case with perturbed triangular noise $\mathcal{T}_D(\x)$, we define $\tilde{p}(\tx) = \sum_{\x} \pdata(\x) \mathcal{T}_D ( \tx |\x)$ and the posterior $q(\x | \tx) = \frac{\pdata(\x) \mathcal{T}_D( \tx | \x)}{\tilde p(\tx)}$, where $\mathcal{T}_D( \tx | \x)=\mathcal{T}_D(\tx-\x)$.

We have
\begin{align}
\label{eq:denoise_posterior}
    \mathbb{E}_{\x}[q(\x|\tilde \x)] &= \sum_{\x\in\{\x_d\in \{\bar{\x}_d, \bar{\x}_d+1\}\}}
    \x q(\x|\tilde \x)\\
    &=  \sum_{\x\in\{\x_d\in \{\bar{\x}_d, \bar{\x}_d+1\}\}}
    \x \frac{\pdata(\x) \mathcal{T}_D( \tx | \x)}{p(\tilde \x)}
\end{align}
Using the fact that $\mathcal{T}_D$ is independent among dimensions, we have 
\begin{align}
\label{eq:triangular_joint}
  \pdata(\x) \mathcal{T}_D( \tx | \x) = \pdata(\x) \prod_{d=1}^{D}
  \bigg(\mathds{1}[\x_d=\bar{\x}_d] * (\bar{\x}_d+1-\tx_d) + 
  \mathds{1}[\x_d=\bar{\x}_d+1] * (\tx_d-\bar{\x}_d)
  \bigg)
\end{align}

and
\begin{align}
\label{eq:triangular_posterior}
   \tilde p(\tx) = \sum_{\x\in\{\x_d\in \{\bar{\x}_d, \bar{\x}_d+1\}\}} \pdata(\x) \prod_{d=1}^{D}
   \bigg(\mathds{1}[\x_d=\bar{\x}_d] * (\bar{\x}_d+1-\tx_d) + 
   \mathds{1}[\x_d=\bar{\x}_d+1] * (\tx_d-\bar{\x}_d)
   \bigg)
\end{align}

Plugging in \Cref{eq:triangular_joint} and \Cref{eq:triangular_posterior}, we have
\begin{align}
\label{eq:denoise_posterior_ratio}
   q(\x|\tilde \x)
    &= \frac{\prod_{d=1}^{D}
   \bigg(\mathds{1}[\x_d=\bar{\x}_d] * (\bar{\x}_d+1-\tx_d) + 
   \mathds{1}[\x_d=\bar{\x}_d+1] * (\tx_d-\bar{\x}_d)
   \bigg)}{\sum_{\y\in\{\y_d\in \{\bar{\x}_d, \bar{\x}_d+1\}\}} \frac{\pdata(\y)}{\pdata(\x)} \prod_{d=1}^{D}
   \bigg(\mathds{1}[\y_d=\bar{\x}_d] * (\bar{\x}_d+1-\tx_d) + 
   \mathds{1}[\y_d=\bar{\x}_d+1] * (\tx_d-\bar{\x}_d)
   \bigg)}. 
\end{align}

\Cref{eq:denoise_posterior_ratio} means that given $\tx$, we can compute the posterior in closed form using density ratios $\frac{\pdata(\y)}{\pdata(\x)}$, which can be evaluated using Concrete scores as long as the graph induced by the neighborhood structure is connected (based on \Cref{thm:completeness}).
Knowing $p(\x|\tx)$ also allows us to sample from $p(\x|\tx)$ to perform denoising.
Similarly, given $q(\x|\tilde \x)$, we can also compute \Cref{eq:denoise_posterior} in closed form.

\paragraph{Recovering Stein Scores}
Given Concrete scores, we can uniquely recover the
Stein score of $\tilde{p}(\x)$, denoted $\textbf{s}(\tx)$, using the following equation, where the $d$-th index of $\textbf{s}(\tx)$ is defined as
\begin{align}
    \textbf{s}(\tx)_d&=\frac{\pdata(\bar{\x}_d+1)-\pdata(\bar{\x}_d)}{\pdata(\bar{\x}_d+1)*(\tx_d-\bar{\x}_d)+\pdata(\bar{\x})*(\bar{\x}_d+1-\x_d)}\\
    &=\frac{\frac{\pdata(\bar{\x}_d+1)}{\pdata(\bar{\x}_d)}-1}{\frac{\pdata(\bar{\x}_d+1)}{\pdata(\bar{\x}_d)}*(\tx_d-\bar{\x}_d)+(\bar{\x}_d+1-\x_d)},
\end{align}
where $\x_d\in\mathbb{Z}$ and $\x_d\in[\x_d, \x_d+1)$.
Thus, the Stein score $\textbf{s}(\tx)$ of $\tilde{p}(\x)$ can be constructed using the density ratio from Concrete scores.
Given the recovered Stein score $\textbf{s}(\tx)$, we can perform Langevin dynamics to sample from $\tilde{p}(\x)$, which gives the smoothed sample in \Cref{fig:1d_examples} (smoothed orange). We can then sample from $q(\x|\tx)$ using \Cref{eq:denoise_posterior_ratio} to perform closed-form denoising, which is shown in \Cref{fig:1d_examples} (orange).

\subsection{Connection to Stein Scores}
In this section, we study a special case of the Concrete score, as described in \Cref{prop:proposition1}, to highlight its connection to both the continuous (Stein) score and existing score matching methods.
Given a $D$-dimensional continuous data distribution $p(\x)$ and $\delta>0$, we define a particular neighborhood structure $\mathcal{N}(\rvx)=\{\rvx_{n_i}\}_{i=1}^{D}$ where
$\x_{n_i}=p(\x+\delta \rve_i)$ and $\rve_i$ is the standard (one-hot) basis vector with the $i$-th element $\rve_i = 1$. 
Then:
\begin{equation}
\label{eq:concrete_directional}
    \frac{\vc_\theta(\x,  \mathcal{N})}{\delta}=\frac{1}{p(\x)}\bigg[
    \frac{p(\x+\delta \rve_1)-p(\x)}{\delta },
    \ldots,
    \frac{p(\x+\delta \rve_D)-p(\x)}{\delta }
    \bigg]^T
\end{equation}
From Eq.~\ref{eq:concrete_directional}, we can make two observations. First, we see that $\bigg[
    \frac{p(\x+\delta \rve_1)-p(\x)}{\delta },
    ...,
    \frac{p(\x+\delta \rve_D)-p(\x)}{\delta }
    \bigg]^T$ approximates the directional derivative of $p(\x)$ via a first-order forward difference operation.
    As a result, the scaled Concrete score function $\frac{\vc_\theta(\x, \cN)}{\delta}$ converges to $\vs(\rvx)$ in the limit of $\delta \to 0$:
\begin{align*}
    \lim_{\delta \to 0} \frac{\vc_\theta(\x,  \mathcal{N})}{\delta}=\frac{1}{p(\x)} \lim_{\delta \to 0} \bigg[
    \frac{p(\x+\delta \rve_1)-p(\x)}{\delta },
    \ldots,
    \frac{p(\x+\delta \rve_D)-p(\x)}{\delta }
    \bigg]^T = \frac{\nabla_{\x} p(\x)}{p(\x)} =\nabla_{\x} \log p(\x)
\end{align*}
which is precisely the definition of the Stein score.

\section{Additional Experimental Results}
\label{app:addtl_exp}
\subsection{Likelihood Evaluation on Discrete Tabular Data}
\label{app:addtl_ll_table}
We present additional results from Section~\ref{exp:tabular}, where we trained a MADE model using standard maximum likelihood. 
We use the training setting as detailed in \Cref{app:experiment}.
As shown in Table~\ref{tab:tabular_likelihood_addtl}, this maximum likelihood baseline serves as an upper bound on performance across all score-based methods. 
We note that as in continuous (Stein) score matching, log-likelihoods and score matching losses are not always correlated even though they both theoretically converge to the optimal solution given infinite model capacity~\cite{song2020sliced}. This is due to practical constraints such as model mis-specification or optimization challenges. We believe that this is also the case for discrete score matching, which explains why directly optimizing likelihood outperforms all other approaches in Table~\ref{tab:tabular_likelihood_addtl}. 
{\begin{table}
\vspace{-10pt}
\newcommand{\mround}[1]{\round{#1}{2}}
    \newcommand{\nastar}{\multicolumn{1}{c}{\hspace{15pt}--~*}}
    \newcommand{\na}{\multicolumn{1}{c}{--}}
    \centering %
    {
        \centering %

        \begin{adjustbox}{max width=0.9\linewidth}
        \begin{tabular}{l @{\extracolsep{5pt}}  c c c c @{}}
        \\
        \toprule
        {Datasets}
        & Ratio Matching $(\uparrow)$
        & Discrete Marginalization $(\uparrow)$
        & CSM (Ours) $(\uparrow)$
        &{Log-Likelihood $(\uparrow)$}
        \\ 
        \midrule
        NLTCS %
        & -6.15 %
        & -6.21 %
        & \textbf{-6.13} %
        & {-6.00} %
        \\
        Plants%
        & -15.44 %
        & -19.03 %
        & \textbf{-14.02} %
         & {-12.52} %
        \\
        Jester %
        & -56.49 %
        & -57.06 %
        & \textbf{-54.91} %
         & {-51.77} %
        
        \\
        Amazon Diaper
        & \textbf{-10.69} %
        & -42.52 %
        & -11.13 %
        & {-9.82} %
        \\
        Amazon Feeding
        & \textbf{-12.09} %
        & -35.96 %
        & -12.65 %
         & {-11.29} %
        \\
        Amazon Gifts
        & -4.57 %
        & -4.28 %
        & \textbf{-4.22} %
         & {-3.43} %
        \\
        Amazon Media
        & \textbf{-10.22} %
        & -13.77 %
        & -10.30 %
         & {-7.79} %
        \\
        Amazon Toys
        & -9.83 %
        & -16.34 %
        & \textbf{-9.30} %
        & {-7.71} \\ %
    \bottomrule
    \end{tabular}
    \end{adjustbox}
}
\caption{
Log-likelihood comparisons on discrete discrete tabular datasets. Higher is better.
We find that CSM demonstrates good performance, almost always outperforming or performing comparably relative to the \texttt{Ratio Matching} and \texttt{Discrete Marginalization} baselines.
}
\label{tab:tabular_likelihood_addtl}
\vspace{-5pt}
\end{table}
}
{
\subsection{Neighborhood Structure Specification in Practice}
In theory, our approach can use any neighborhood structure as long as the neighborhood-induced graph is connected (see \Cref{thm:completeness}). 
In practice, the choice of neighborhood structure can affect performance due to challenges in optimization and proper model specification.
We provide an empirical analysis to build intuition on a series of 1-D synthetic datasets, and believe that the same intuition can generalize to higher dimensions.\\

First, we consider the 1-D distribution in \Cref{fig:app_1d_neighbor_high}, where the data distribution does not contain any low density regions (regions with close to zero density). 
We parameterize the 
unnormalized probability distribution with softmax logits,
and train the model using CSM. 
For the training objective, we explore 4 different neighborhood structures (3 different cycles as well as a fully-connected graph) as shown in \Cref{fig:app_1d_neighbor_high}. 
In this setting, we observe that different neighborhood structures yield similar performances as evaluated by log-likelihood (see \Cref{fig:app_1d_neighbor_high}).
\begin{figure}[H]
    \centering
    \includegraphics[width=\linewidth]{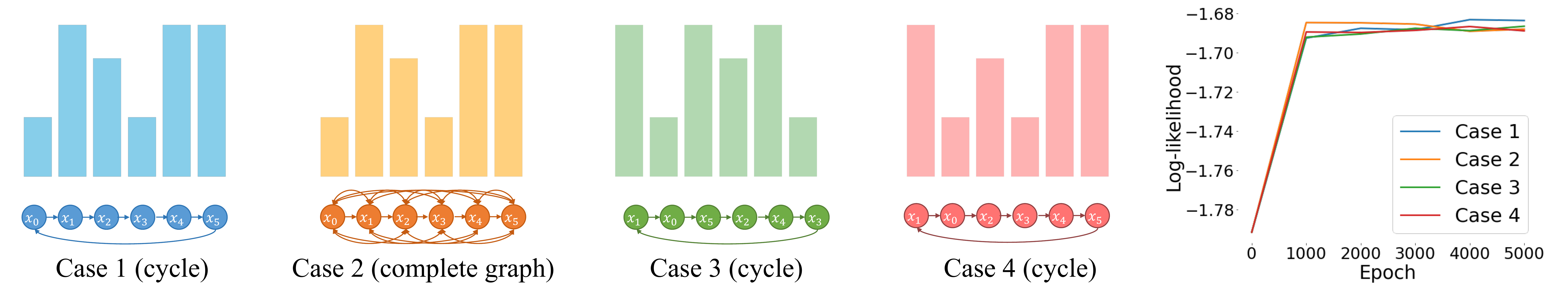}
    \caption{
    {
    Examples of neighborhood structures and their corresponding log-likelihood values (higher is better) when trained with \model{}.}
    }
    \label{fig:app_1d_neighbor_high}
\end{figure}
Next, we consider a different 1-D data distribution which contains low density regions. As shown in \Cref{fig:app_1d_neighbor_low}, we observe that the neighborhood structure in this setting plays a critical role in the final log-likelihoods: the complete graph in Case 2 outperforms all other structures.
Interestingly, when the low-density regions are in between sets of high-density modes (as in Case 3 in \Cref{fig:app_1d_neighbor_low}), we find that the model can still perform comparably relative to Case 2.
\begin{figure}[H]
    \centering
    \includegraphics[width=\linewidth]{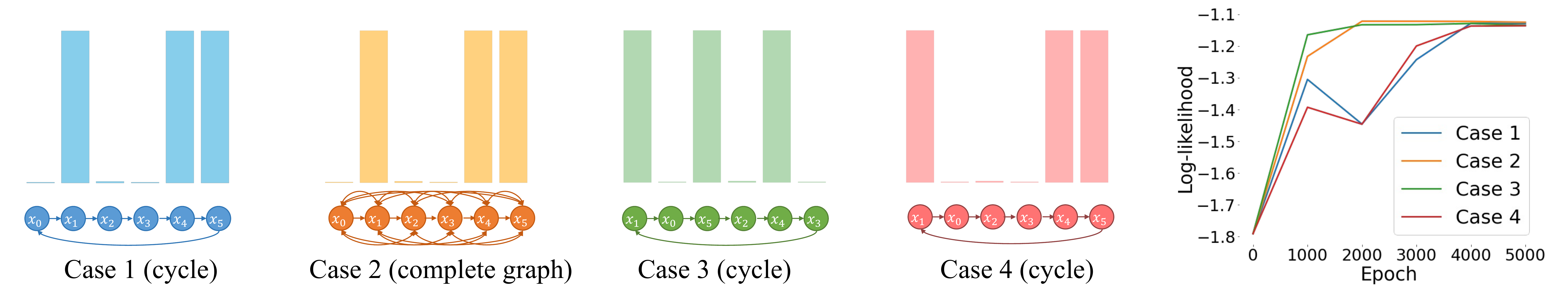}
    \caption{
    {
    Examples of neighborhood structures and their corresponding log-likelihood values (higher is better) when trained with CSM.}
    }
    \label{fig:app_1d_neighbor_low}
\end{figure}
This observation is similar to that of Stein score matching~\cite{song2019generative,hyvarinen2005estimation}. Therefore, we believe that a similar noise annealing procedure followed by denoising CSM may potentially alleviate the effects of poorly chosen neighborhood structures for data distributions with low density regions.
We will leave this investigation for future work.
Such empirical results highlight that the best-performing structure in practice is often data-dependent. That is, a neighborhood graph which respects the particular dataset structure will perform the best out of all possible graph structures (analogous to the way in particular types of inductive biases are useful for solving specific tasks). \\
One interesting avenue for future research would be to draw inspiration from \cite{liu2021gradient} and investigate whether it is possible to construct an ``optimal'' neighborhood structure for a given dataset. In this work, we uniformly sample a set of neighbors from $\cN(\rvx)$ when computing the CSM objective in practice (Algorithms 1 and 2), which can lead to issues with high variance during training. We hypothesize that an optimal proposal distribution over $\cN(\rvx)$ that minimizes variance may lead to insights on what the optimal neighborhood structure should be.
\subsection{Denoising Concrete Score Matching}
\label{app:dcsm_addtl}
We present additional experiments on denoising concrete score matching (D-CSM) as introduced in \Cref{sec:denoising}. In particular, we consider two 2-D toy benchmark datasets as commonly used in the density estimation literature~\cite{grathwohl2018ffjord,de2020block}. 
We quantize the data into $91\times91$ equally-distanced bins to obtain the discrete training data (see \texttt{Clean data} in \Cref{fig:app:extra_denoisng_csm_2d}). Similar to our previous experiments, we use the grid neighborhood structure as detailed in \Cref{2d_multi_path} for training via D-CSM. \\

For the discrete noise distribution $\tilde{q}(\tx | \x)$, we consider the following distribution where each marginal is an independent 1-D Categorical distribution defined as:
\begin{equation}
\tilde{q}(\tx | \x)_i=
\begin{cases}
 w & \text{when $\tx_i=\x_i$}\\
 \frac{1- w}{90} & \text{when $\tx_i\neq\x_i$},\\
\end{cases}
\end{equation}
where $0<w<1$ and $\tilde{q}(\tx | \x)_i$ denotes the $i$-th entry of $\tilde{q}(\tx | \x)$. Note that higher values of $w$ correspond to smaller noise levels. We use the same experimental setting as detailed in \Cref{app:sec:2d_experiments}, where we parameterize the un-normalized probability distribution with softmax logits, but train the model using D-CSM (see \Cref{eq:denoising_csm}).\\

We present the results in \Cref{fig:app:extra_denoisng_csm_2d}. We observe that samples from our model (\textcolor{green}{green}) matches samples from the ground truth noisy data distributions $\tilde p(\tx)$ (\textcolor{blue}{blue}), indicating the practical effectiveness of D-CSM. We leave a more in-depth exploration of the D-CSM approach for future work.
\begin{figure}[H]
    \centering
    \includegraphics[width=\linewidth]{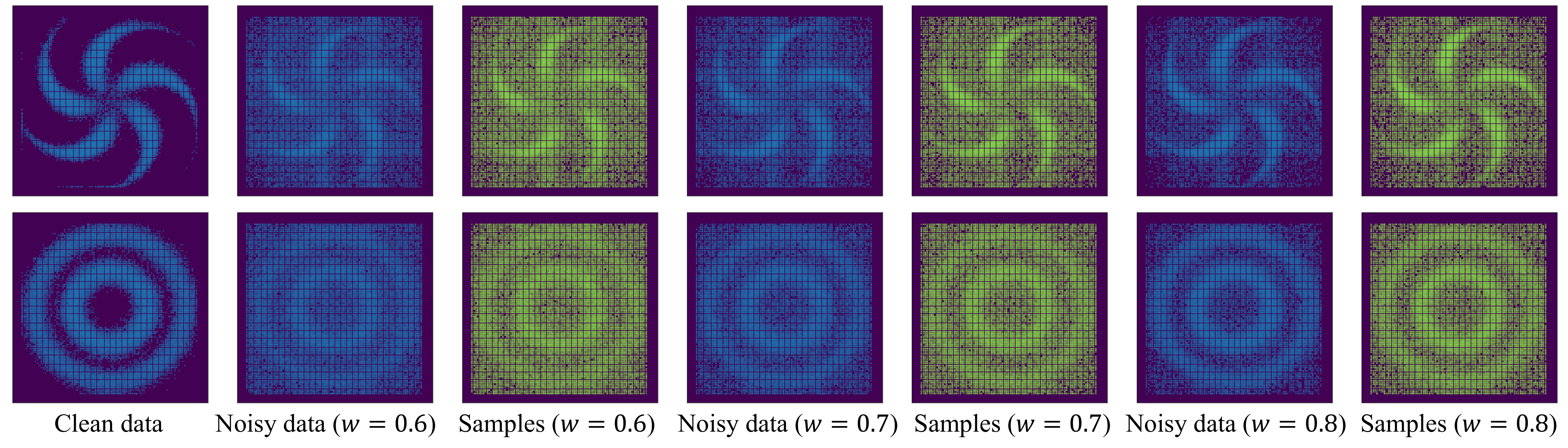}
    \caption{
    {Samples from models trained with D-CSM. The blue samples are the ground truth samples from $\tilde{p}(\tx)$ obtained by perturbing the clean data distribution with $\tilde{q}(\tx | \x)$. As we increase $w$, the variance of the noise distribution $\tilde{q}(\tx | \x)$ decreases and the perturbed data distribution $\tilde{p}(\tx)$ becomes less noisy. The samples from our models (\textcolor{green}{green}) match the ground truth perturbed data samples (\textcolor{blue}{blue}) across different values of $w$, demonstrating the effectiveness of D-CSM.}
    }
    \label{fig:app:extra_denoisng_csm_2d}
\end{figure}
}

\section{Experimental Details}
\label{app:experiment}
In this section, we provide more details for each of the experimental settings. We will release our code upon publication.

\subsection{1-D Discrete Data}
We discuss the experimental setting for \Cref{sec:1d_data}. 
We use a 16-category 1-D data distribution with class labels from 0 to 15, and normalize the data to $[0,1]$ before feeding it into the model. 
We use a 3-layer shallow MLP for the score network with 100 hidden units and Tanh activations. 
We use the model to directly parameterize the Concrete score and define the neighborhood structure to be a Cycle structure as shown in \Cref{fig:graph_simple}. 
We train the model using the Adam optimizer with learning rate $0.001$ for 100,000 iterations. The Langevin samples and denoised samples in \Cref{fig:1d_examples} are obtained using the approach detailed in \Cref{app:langevin_noise_perturbation}.

\subsection{2-D Multi-Class Datasets}
\label{app:sec:2d_experiments}
For the 2-D experiments, we consider three 2-D toy benchmark datasets with multiple modes and discontinuities as commonly used in the density estimation literature \cite{grathwohl2018ffjord,de2020block}. 
We quantize the data into $91\times91$ equally-distanced bins to obtain the discrete training  data (see \Cref{fig:2d_samples}). We use the grid neighborhood structure for training via CSM (see \Cref{fig:graph_simple}). 
For a fair comparison, we use the same model architecture and training configurations across all methods. We directly parameterize the unnormalized probability distribution with softmax logits, but train the model using the baselines and CSM. Since there are  $91\times91$ possible classes, our model directly parameterizes an unnormalized distribution with $91\times91$ possible values. 
We initialize the model to have a uniform probability across classes and train the model using the Adam optimizer with learning rate 0.0005. To sample from the model, we can either use likelihood sampling (by normalizing the model) or Metropolis-Hastings. We empirically observe that both approaches work well.

\subsection{Discrete Tabular Data}
For the tabular experiments, 
we consider tabular (discrete) datasets drawn from both the Twenty Datasets \cite{van2012markov} and the Amazon Baby Registries benchmarks \cite{gillenwater2014expectation}. 
In order to evaluate likelihoods, we directly parameterize the probability distribution with a discrete autoregressive model (MADE~\cite{germain2015made}), but train the model using the baseline approaches and CSM.
The MADE model uses 2 residual blocks with latent dimension 100 and Tanh activations. We train the model for 50000 iterations using the Adam optimizer with learning rate 0.0005.
We use a grid neighborhood structure for CSM (see \Cref{fig:graph_simple}).
For a fair comparison, we use the same model architecture and experimental configurations across all methods. 
Similar to our previous experiments, we use the grid neighborhood structure for training via CSM (see \Cref{fig:graph_simple}).

\subsection{Discrete Image Data}
We experiment with the binarized MNIST \cite{lecun1998mnist} dataset, which has 784 dimensions.
We perturb the images $\x$ with binarized Gaussian noise $q(\tx|\x)=\mathcal{N}(\tx|\x,\sigma_i^2I)$ ($i=1,...,7$). Specifically, given a sample from the noisy data distribution $\tx \sim q(\tx|\x)$, we discretize $\tx$ by mapping it to its nearest integer neighbor, 
and then take the modula by 2 to map $\tx$ into binary bins. We note that this process is similar to applying Categorical noise. 
In our experiment, we use $\{\sigma_1,\sigma_2,\sigma_3, \sigma_4, \sigma_5, \sigma_6, \sigma_7\}=\{0.1, 0.2, 0.25, 0.3, 0.4, 0.5, 0.65\}$. 
We normalize the data to lie in between $[-1,1]$ before feeding it into the model.
We directly parameterize the Concrete score function with a U-Net~\cite{ronneberger2015u}, and train the model using CSM based on a grid neighborhood structure. 
We train the models until convergence with the Adam optimizer using learning rate 0.0002.

We use Metropolis-Hastings for sampling from the model after training: in particular, we initialize the sample to be drawn from uniform binary noise then use the model corresponding to noise level $\sigma_7$ to perform Metropolis-Hastings updates. 
After that, we use the output from the model to initialize the input for sampling from the model with noise level $\sigma_6$. We then repeat the procedure and initialize the model corresponding to the previous noise level with the output from the next noise level---a process similar to \cite{song2019generative}. We observe that each noise level requires around 100-800 steps to obtain reasonable results.

\subsection{Training for Special Neighborhood Structures}
\label{app:train_neighborhood}

\begin{figure}[ht]
\centering
\includegraphics[width=0.5\linewidth]{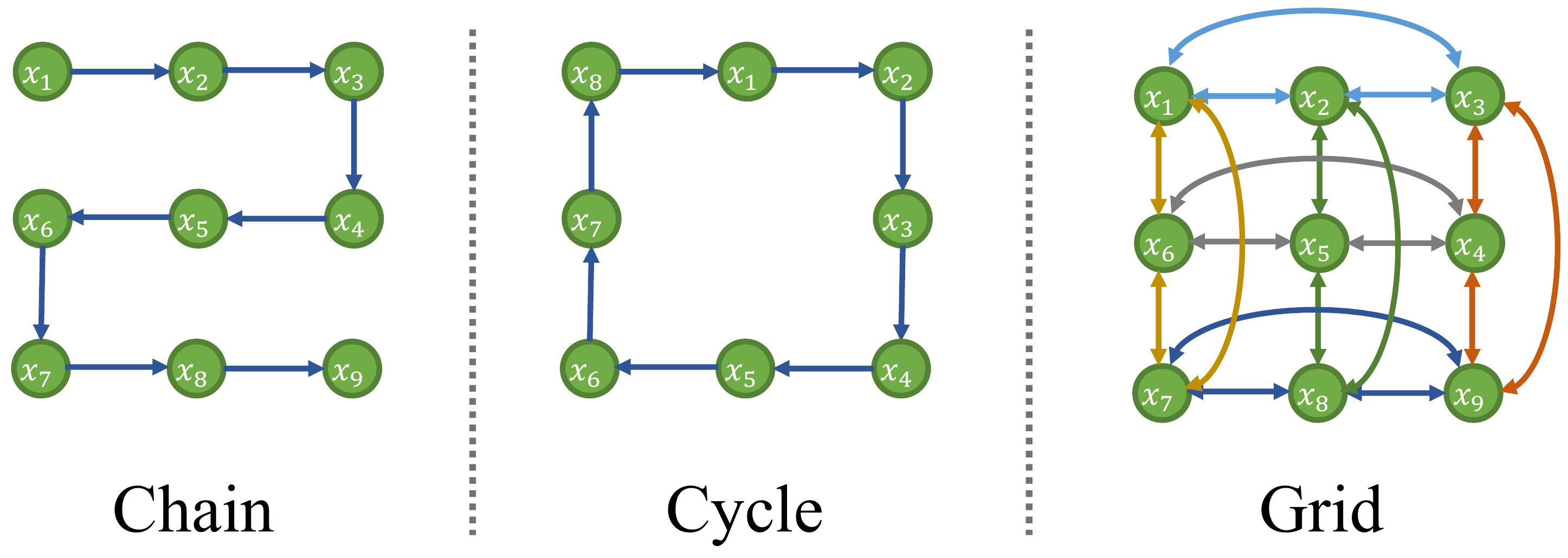}
\caption{Special types of neighborhood structures considered in the paper. }
\label{fig:graph_simple}
\end{figure}

In this section, we provide additional details on special neighborhood structures where we can design even more efficient training procedures than using \Cref{alg:term1} and \Cref{alg:term2}.
For simplicity, we name the elements in $\mathcal{X}$: $\{\x_1,\x_2,..., \x_N\}$.

\paragraph{Chain.}
When the neighborhood-induced graph is a Chain (\Cref{fig:graph_simple}), the second term $\cJ_2$ in \Cref{eq:discrete_sm_simplified} simplifies to the following via reparameterization:
\begin{align}
    \sum_{k=1}^{N-1} 2\pdata(\x_{k+1}) \vc_\theta(\x_k; \mathcal{N})_1
    =\sum_{k=2}^{N} 2\pdata(\x_{k}) \vc_\theta(\x_{k-1}; \mathcal{N})_1.
\end{align}
Note that the subscript of $\vc_\theta$ is always one since each node has at most one neighbor.

\paragraph{Cycle.}
Similarly, when the neighborhood-induced graph is a Cycle (\Cref{fig:graph_simple}), the second term $\cJ_2$ becomes:
\begin{align}
    \sum_{k=1}^{N} 2\pdata(\x_{k+1}) \vc_\theta(\x_k; \mathcal{N})_1
    =\sum_{k=1}^{N} 2\pdata(\x_{k}) \vc_\theta(\x_{(k-1\text{ mod}N)}; \mathcal{N})_1
\end{align}
Thus for both Chains and Cycles, the term $\cJ_2$ in \Cref{eq:discrete_sm_simplified} can be evaluated efficiently using samples from $\pdata(\rvx)$.

\paragraph{Grid.}
When the neighborhood-induced graph is a Grid (\Cref{fig:graph_simple}), the second term $\cJ_2$ can be approximated efficiently by decomposing each dimension into smaller Cycles. Specifically, to estimate $\cJ_2$, we can uniformly sample a dimension $d$, and then use the same reparameterization approach used for the Cycle structure at each dimension.

\newpage
\section{Additional Related Works and Corrections}
\label{app:addtl_works}
In this section, we provide additional details about the baseline methods: in particular, we found that the original equations in \cite{lyu2012interpretation} were incorrect. We elaborate upon them below.

\subsection{Ratio Matching with Discrete Data}
We assume with slight abuse of notation that $\cX \in \{0,1\}^D$, and $\pdata(\rvx)$ is the unknown discrete data distribution.
Ratio matching \cite{hyvarinen2007connections,lyu2012interpretation} was originally proposed as an alternative approach to score matching for learning discrete binary probability distributions from samples. 
Similar to score matching, it leverages the fact that ratios of probabilities are also independent of the intractable normalizing constant (due to cancellation), and seeks to match the ground truth density ratios $\frac{\pdata(\rvx)}{\pdata(\rvx_{-d})} = \frac{q_\theta(\rvx)}{q_\theta(\rvx_{-d})}$ where $\rvx_{-d}$ denotes the vector $\rvx$ with the $d$th entry bit-flipped (e.g. from 0 to 1). 
Concretely, ratio matching minimizes the following objective: 
\begin{align}
\label{eq:rm_objective}
\resizebox{\linewidth}{!}{
    $\mathcal{J}_{RM}(\theta) = \frac{1}{2}\bbE_{\pdata(\x)}\bigg[ \sum_{d=1}^D \bigg( g \bigg(\frac{\pdata(\rvx)}{\pdata(\rvx_{-d})} \bigg) - g \bigg(\frac{q_\theta(\rvx)}{q_\theta(\rvx_{-d})} \bigg) \bigg)^2 + \bigg( g \bigg(\frac{\pdata(\rvx_{-d})}{\pdata(\rvx)} \bigg) - g \bigg(\frac{q_\theta(\rvx_{-d})}{q_\theta(\rvx)} \bigg) \bigg)^2 \bigg] $
}
\end{align}
where $g(z) = \frac{1}{1+z}$ for $z \in \bbR^+$ is a nonlinear transformation of the ratios for numerical stability. The symmetrized objective tries to minimize the squared distance between the two density ratios.

As the original ratio matching approach~\cite{hyvarinen2007connections,lyu2012interpretation} is proposed for binary data, \cite{lyu2012interpretation} extends ratio matching to multi-class data. 
Following the notation in \cite{lyu2012interpretation}, let $\x^{\backslash i}$ denote the vector formed by dropping the $i$-th element (which we call $\x_i$) from $\x$.
To make things concrete, this means that $p(\rvx) = p(\x^{\backslash i}, \x_i)$.
We also let $q_\theta(\xi_i|\x^{\backslash i})$ denote the conditional probability for the $i$-th index of $\x$ taking the value $\xi_i$, and let $q_\theta(\sim\xi_i|\x^{\backslash i})$ denote the conditional probability for the $i$-th index of $\x$ not taking the value $\xi_i$. 

The generalized ratio matching objective proposed in \cite{lyu2012interpretation} is:
\begin{align}
    \theta^{*}&=\argmin_{\theta}\sum_{\x}\pdata(\x)\sum_{i=1}^{d}\sum_{\xi_i}\bigg[g\bigg(\frac{p(\xi_i,\x^{\backslash i})}{p(\sim \xi_i,\x^{\backslash i})}\bigg)-g\bigg(\frac{q_\theta(\xi_i,\x^{\backslash i})}{q_{\theta}(\sim \xi_i,\x^{\backslash i})}\bigg)\bigg]^2
\end{align}
which can be simplified as:
\begin{equation}
\label{eq:wrong_rm}
    \theta^{*}= \argmin_{\theta} \sum_{\x}\pdata(\x)\sum_{i=1}^{d}\sum_{\xi_i}(1-q_\theta(\xi_i|\x^{\backslash i}))^2.
\end{equation}
It is easy to see that the optimal $\theta^{*}$ for \Cref{eq:wrong_rm} can be achieved independent of $\pdata(\x)$, which is problematic. We also empirically show that model trained with \Cref{eq:wrong_rm} cannot learn meaningful features based on the data (see \texttt{Ratio} in \Cref{fig:2d_samples}).

As an alternative, we provide a corrected multi-class ratio matching objective. Similar to \cite{hyvarinen2007connections,hyvarinen2007some}, we assume that all the probabilities are non-zero.
For simplicity, we use $p(\x)$ to denote $\pdata(\x)$.
The objective for multi-class ratio matching can achieved by optimizing:
\begin{align}
\label{eq:ratio_matching_fixed}
    \theta^{*}&=\argmin_{\theta}\sum_{\x}p(\x)\sum_{i=1}^{d}\sum_{\xi_i}\bigg[g\bigg(\frac{p(\xi_i,\x^{\backslash i})}{p(\sim \xi_i,\x^{\backslash i})}\bigg)-g\bigg(\frac{q_\theta(\xi_i,\x^{\backslash i})}{q_{\theta}(\sim \xi_i,\x^{\backslash i})}\bigg)\bigg]^2\\
    &=\argmin_{\theta}\sum_{\x}p(\x)\sum_{i=1}^{d}\bigg[(1-q_\theta(\x_i|\x^{\backslash i}))^2+\sum_{\xi_i\neq \x_i}q_\theta(\xi_i|\x^{\backslash i})^2\bigg].
\end{align}

\begin{proof}
\begin{align*}
       \theta^{*}&=\argmin_{\theta}\sum_{\x}p(\x)\sum_{i=1}^{d}\sum_{\xi_i}\bigg[g\bigg(\frac{p(\xi_i,\x^{\backslash i})}{p(\sim \xi_i,\x^{\backslash i})}\bigg)-g\bigg(\frac{q_\theta(\xi_i,\x^{\backslash i})}{q_{\theta}(\sim \xi_i,\x^{\backslash i})}\bigg)\bigg]^2\\
    &=\argmin_{\theta}\sum_{\x}p(\x)\sum_{i=1}^{d}\sum_{\xi_i}\bigg[p(\sim \xi_i|\x^{\backslash i})-q_\theta(\sim\xi_i|\x^{\backslash i})\bigg]^2\\
    &=\argmin_{\theta}\sum_{\x}p(\x)\sum_{i=1}^{d}\sum_{\xi_i}\bigg[p(\xi_i|\x^{\backslash i})-q_\theta(\xi_i|\x^{\backslash i})\bigg]^2\\
    &=\argmin_{\theta}\sum_{\x}p(\x)\sum_{i=1}^{d}\sum_{\xi_i}\bigg[q^2_\theta(\xi_i|\x^{\backslash i})-2p( \xi_i|\x^{\backslash i})q_\theta(\xi_i|\x^{\backslash i})\bigg]^2\\
    &=\argmin_{\theta}\sum_{\x}p(\x)\sum_{i=1}^{d}\sum_{\xi_i}q^2_\theta(\xi_i|\x^{\backslash i}) -\sum_{\x}p(\x)\sum_{i=1}^{d}\sum_{\xi_i}2p( \xi_i|\x^{\backslash i})q_\theta(\xi_i|\x^{\backslash i}) \\
    &=\argmin_{\theta}\sum_{\x}p(\x)\sum_{i=1}^{d}\sum_{\xi_i}q^2_\theta(\xi_i|\x^{\backslash i}) -\sum_{i=1}^{d}\sum_{\x^{\backslash i}, \x_i}p(\x)\sum_{\xi_i}2p( \xi_i|\x^{\backslash i})q_\theta(\xi_i|\x^{\backslash i}) \\
    &=\argmin_{\theta}\sum_{\x}p(\x)\sum_{i=1}^{d}\sum_{\xi_i}q^2_\theta(\xi_i|\x^{\backslash i}) -\sum_{i=1}^{d}\sum_{\x^{\backslash i}, \xi_i}2p( \xi_i,\x^{\backslash i})q_\theta(\xi_i|\x^{\backslash i})\\
    &=\argmin_{\theta}\sum_{\x}p(\x)\sum_{i=1}^{d}\sum_{\xi_i}q^2_\theta(\xi_i|\x^{\backslash i}) -\sum_{i=1}^{d}\sum_{\x}2p( \x)q_\theta(\x_i|\x^{\backslash i}) \\
    &=\argmin_{\theta}\sum_{\x}p(\x)\sum_{i=1}^{d}\frac{\sum_{\xi_i}q^2_\theta(\xi_i,\x^{\backslash i}) +2q_\theta(\x_i,\x^{\backslash i})q_{\theta}(\x^{\backslash i})}{q^2_{\theta}(\x^{\backslash i})}\\
    &=\argmin_{\theta}\sum_{\x}p(\x)\sum_{i=1}^{d}\frac{-q^2_{\theta}(\x^{\backslash i})+ \sum_{\xi_i\neq\x_i}q^2_\theta(\xi_i,\x^{\backslash i}) +(\sum_{\xi_i\neq\x_i} q_\theta(\xi_i,\x^{\backslash i}))^2}{q^2_{\theta}(\x^{\backslash i})}\\
    &=\argmin_{\theta}\sum_{\x}p(\x)\sum_{i=1}^{d}\frac{ \sum_{\xi_i\neq\x_i}q^2_\theta(\xi_i,\x^{\backslash i}) +(q_\theta(\x^{\backslash i})-q_\theta(\x_i,\x^{\backslash i}))^2}{q^2_{\theta}(\x^{\backslash i})}\\
    &=\argmin_{\theta}\sum_{\x}p(\x)\sum_{i=1}^{d}\bigg[(1-q_\theta(\x_i|\x^{\backslash i}))^2+\sum_{\xi_i\neq\x_i}q^2_\theta(\xi_i|\x^{\backslash i}) \bigg]. 
\end{align*}

\end{proof}

We provide the samples trained with \Cref{eq:ratio_matching_fixed} in \Cref{fig:2d_samples} (\texttt{Ratio-fixed}). We observe that the samples appear to be more reasonable then those from a model trained with \Cref{eq:wrong_rm} (\texttt{Ratio}).

\subsection{Discrete Marginalization}
Given a multi-class discrete data distribution $\pdata(\x)$, discrete marginalization~\cite{lyu2012interpretation} proposes to learn the data distribution by minimizing
\begin{align}
\label{eq:discrete_marginalization}
    \sum_{\x}\pdata(\x)\sum_{i=1}^{d}\bigg[\bigg(\frac{\mathcal{M}_i p}{p}\bigg)^2+\bigg(\frac{\mathcal{M}_{i} q_{\theta}}{q_{\theta}}\bigg)^2-2\mathcal{M}_i\bigg(\frac{\mathcal{M}_i q_{\theta}}{q_{\theta}}\bigg)\bigg],
\end{align}
where $\mathcal{M}_i$ is the marginalization operator defined as
$\mathcal{M}_{i}: \mathcal{F}^1 \to\mathcal{F}^{1}: f(\x)\to \int_{\x_i}f(\x)d\x_i$ or $\sum_{\x_i}f(\x)$ in the discrete case. Thus $\frac{\mathcal{M}_i  p}{p}=\frac{1}{\pdata(\x_i|\x^{\backslash i})}$ and $\frac{\mathcal{M}_i  q_{\theta}}{q_{\theta}}=\frac{1}{q_{\theta}(\x_i|\x^{\backslash i})}$.

The authors further simplified \Cref{eq:discrete_marginalization} to
\begin{align}
\label{eq:discrete_marginalization_simplified}
    \theta^{*}=\argmin_{\theta}\sum_{\x}\pdata(\x)\sum_{i=1}^{d}\sum_{\xi_i}\frac{1-2q_{\theta}(\xi_i|\x^{\backslash i})}{q_{\theta}^2(\xi_i|\x^{\backslash i})}.
\end{align}
However, we observe that this is a typo in \Cref{eq:discrete_marginalization_simplified}---the optimal $\theta^{*}$ can be achieved while being independent of $\pdata(\x)$. In our experiments, we also observe that model trained with \Cref{eq:discrete_marginalization_simplified} cannot learn meaningful representations of the data (see \texttt{Marginal} in \Cref{fig:2d_samples}).

In the following, we provide the corrected simplified version of \Cref{eq:discrete_marginalization}.
\begin{theorem}
Optimizing \Cref{eq:discrete_marginalization} is equivalent to optimizing
\begin{align}
\label{eq:discrete_marginalization_fixed}
     \theta^{*}=\argmin_{\theta}\sum_{\x}\pdata(\x)\sum_{i=1}^{d}\bigg[\frac{1}{q_{\theta}^2(\x_i|\x^{\backslash i})}-\sum_{\xi_i}\frac{2}{q_{\theta}(\xi_i|\x^{\backslash i})}\bigg].
\end{align}
\end{theorem}

\begin{proof}
\begin{align}
 \theta^{*}&=\argmin_{\theta}\sum_{\x}\pdata(\x)\sum_{i=1}^{d}\bigg[\bigg(\frac{\mathcal{M}_i p}{p}\bigg)^2+\bigg(\frac{\mathcal{M}_{i} q_{\theta}}{q_{\theta}}\bigg)^2-2\mathcal{M}_i\bigg(\frac{\mathcal{M}_i q_{\theta}}{q_{\theta}}\bigg)\bigg]\\
 &=\argmin_{\theta}\sum_{\x}\pdata(\x)\sum_{i=1}^{d}\bigg[\bigg(\frac{\mathcal{M}_{i} q_{\theta}}{q_{\theta}}\bigg)^2-2\mathcal{M}_i\bigg(\frac{\mathcal{M}_i q_{\theta}}{q_{\theta}}\bigg)\bigg]\\
     &=\argmin_{\theta}\sum_{\x}\pdata(\x)\sum_{i=1}^{d}\bigg[\frac{1}{q_{\theta}^2(\x_i|\x^{\backslash i})}-\sum_{\xi_i}\frac{2}{q_{\theta}(\xi_i|\x^{\backslash i})}\bigg].
\end{align}
In the last equation, we used the fact that $\sum_{\x}\pdata(\x) \mathcal{M}_i(\frac{\mathcal{M}_i q_{\theta}}{q_{\theta}})=\sum_{\x}\pdata(\x)\sum_{\xi_i}(\frac{\mathcal{M}_i q_{\theta}}{q_{\theta}})=\sum_{\x}\pdata(\x)\sum_{\xi_i}\frac{1}{q_{\theta}(\xi_i|\x^{\backslash i})}$, which directly follows from Lemma 4 in \cite{lyu2012interpretation}.

\end{proof}

We provide the samples trained with \Cref{eq:discrete_marginalization_fixed} in \Cref{fig:2d_samples} (\texttt{Marginal-fixed}). We observe that the samples appear to be more reasonable then those from a model trained with \Cref{eq:discrete_marginalization_simplified} (\texttt{Marginal}).

\end{document}